\newif\ifarxiv
\arxivtrue

\documentclass[11pt]{article}
\usepackage{dsfont}

\usepackage{tikz}
\usetikzlibrary{arrows.meta, positioning}
\usepackage{pgfplots}
\pgfplotsset{compat = newest}
\usetikzlibrary{decorations.pathreplacing,calligraphy}

\usepackage{float}
\usepackage{amssymb,amsmath,amsthm}
\usepackage{bm}
\usepackage[noend]{algpseudocode}
\usepackage{graphicx}
\usepackage{verbatim}%for comment blocks
\usepackage{url,xspace,hyperref}
\usepackage{makecell}
\usepackage{caption}
\usepackage{subcaption}
\usepackage{multirow}
\usepackage{multicol}
\usepackage{latexsym}
\usepackage{amsmath,amssymb,enumerate}
\usepackage[final]{showlabels}
\usepackage{xspace}
\usepackage{algorithm,algorithmicx}
\usepackage{algpseudocode}
\usepackage{float}
\usepackage{xcolor}
\usepackage{mathrsfs}
\usepackage{cleveref}
\usepackage{nicefrac}
\usepackage{natbib}
\usepackage{enumitem}
\usepackage{thmtools}
\definecolor{mydarkblue}{rgb}{0,0.08,0.45}
\hypersetup{ %
  pdftitle={},
  pdfsubject={Proceedings of the International Conference on Machine Learning 2026},
  pdfkeywords={},
  pdfborder=0 0 0,
  pdfpagemode=UseNone,
  colorlinks=true,
  linkcolor=mydarkblue,
  citecolor=mydarkblue,
  filecolor=mydarkblue,
  urlcolor=mydarkblue,
}

\usepackage{cancel}
\usepackage{booktabs}

\usepackage{geometry}
\newtheorem{assumption}{Assumption}[section]
\newtheorem{theorem}{Theorem}[section]
\newtheorem*{theorem*}{Theorem}

\newtheorem{lemma}[theorem]{Lemma}
\newtheorem{observation}[theorem]{Observation}
\newtheorem{corollary}[theorem]{Corollary}

\def\halpha{\hat{\alpha}}
\def\Alg{\ensuremath{\mathsf{ExpRef}}}
\def\opt{\ensuremath{\mathsf{Opt}}}
\def\ba{\mathbf{a}}

\def\bhalpha{\boldsymbol{\hat{\alpha}}} 
\def\E{\mathbb{E}}

\def\e{e}
\def\H{H}

\def\pr{\Pr\nolimits}
\def\KL{\mathrm{KL}}
\def\Bern{\mathrm{Bern}}
\def\TO{\tilde{O}}
\DeclareMathOperator{\poly}{poly}

\newcommand{\ignore}[1]{}

\newcounter{mynote}[section]

\title{Learning Markov Decision Processes under Fully Bandit Feedback}
\author{Zhengjia Zhuo\thanks{Department of Industrial and Operations Engineering, University of Michigan, Ann Arbor, USA. Research supported in part by NSF grant  CCF-2418495.} \and Anupam Gupta\thanks{Computer Science Department, New York University,
    New York NY. Supported in part by NSF awards CCF-2224718 and CCF-2422926.} \and Viswanath Nagarajan$^*$}

\renewcommand{\epsilon}{\varepsilon}
\begin{document}

\maketitle

\begin{abstract}

A standard assumption in Reinforcement Learning is that the agent observes every visited state-action pair in the associated Markov Decision Process (MDP), along with the per-step rewards. Strong theoretical results are known in this setting, achieving nearly-tight $\Theta(\sqrt{T})$-regret bounds.  However, such detailed feedback can be unrealistic, and recent research has investigated more restricted settings such as trajectory feedback, where the agent observes all the visited state-action pairs, but only a single \emph{aggregate} reward. In this paper, we consider a far more restrictive ``fully bandit'' feedback model for episodic MDPs, where the agent does not even observe the visited state-action pairs---it only learns the aggregate reward.  We provide the first efficient bandit learning algorithm for episodic MDPs with $\widetilde{O}(\sqrt{T})$ regret. Our regret has an   exponential dependence on the horizon length $\H$, which we show is necessary.    We also  obtain  improved nearly-tight regret bounds for ``ordered'' MDPs; these can be used to model classical stochastic optimization problems such as $k$-item prophet inequality and sequential posted pricing. Finally, we evaluate the empirical performance of our algorithm for the setting of  $k$-item prophet inequalities; 
despite the highly restricted feedback, our algorithm's performance is comparable to that of a state-of-art learning algorithm (UCB-VI) with detailed state-action feedback. 

\end{abstract}

\section{Introduction}
\label{sec:introduction}

Markov Decision Processes (MDPs) are a fundamental model in
Reinforcement Learning, where a learning agent interacts with an
unknown environment over time. This has been applied to many areas
such as robotics, gaming, finance and autonomous driving~\citep{SuttonB2018,SilverHMGSDSAPL16,YangLZW20,KiranSTMSYP22,TangAHCMS25}. A MDP
consists of a collection of states and actions with a reward
associated with each state-action pair. The agent wants to
select an action at every state so as to maximize its expected total
reward. In reinforcement learning, the rewards and transition
probabilities of the MDP are unknown, and the agent needs to learn
these through repeated interactions.  Its performance is measured via its regret, the difference between the expected
long-term reward of the algorithm and the optimal policy.

The classic setting of MDP learning assumes that the agent receives as
feedback the reward at {\em every} state-action pair that it
encounters. There has been extensive work in this setting, with the
UCB-VI algorithm achieving a nearly tight regret
bound~\citep{AzarOM17,JakOrAu2010}. 
However, this kind of feedback is not always practical: it is often
costly/tedious to collect such detailed feedback for every
state-action pair. In order to address this drawback, several recent
papers have designed MDP learning algorithms under more restricted
feedback. One example is that of {\em trajectory} feedback
\citep{EfroniMM21,ChatterjiPBJ21,CohenKKM21}, where the agent observes
the trajectory---i.e., the entire sequence of state-action pairs
visited in the run of the policy---but obtains only the
\emph{aggregate} reward, instead of the reward at each state-action
pair.

In this paper, we study MDP learning algorithms in the far more
restrictive {\em fully bandit} feedback model.\footnote{We use the
  term ``fully bandit'' to distinguish our feedback model from the
  more-standard ``semi-bandit'' feedback, where the reward at each
  state-action pair is observed.} In our setting, the agent does not
even observe the trajectory---it merely receives the aggregate reward
as feedback. Our main results are the development of policies that
obtain nearly-optimal regret bounds for episodic MDPs in this
setting.

Furthermore, we consider a special class of ``ordered''
episodic MDPs which arise in many stochastic selection problems, such
as the \emph{prophet inequality}, \emph{sequential posted pricing} and
\emph{online knapsack}. In these MDPs, there is a total ordering of states such that all transitions are monotone w.r.t.\ this ordering; this typically arises from a capacity constraint (e.g., select at most $k$ items). For  ordered MDPs, we obtain
improved regret bounds, which we show are also
near-optimal.

Our main contributions are:
\begin{enumerate}
\item \emph{(Episodic MDPs under Bandit Feedback)} For the setting of episodic MDPs under
  \emph{fully bandit feedback}, we give an online learning algorithm
  that incurs an optimal $\tilde{O}(\sqrt{T})$ regret, where $T$ is
  the number of episodes. In more detail, the regret is roughly
  $(Ak)^{\H}\sqrt{T \log T}$, where $A$ is the number of actions in each
  state, $\H$ is the horizon of each episode, and $k$ is the ``width''
  (i.e., the number of states in each stage), and $T$ is the number of
  episodes. 
  Our algorithm is based on a successive-elimination
  approach~\citep{AuPeOr2010}, where in each phase we explore all actions in an ``active
  set'' and then refine the active sets using empirical maximizers.
\item \emph{(Lower Bound)} We show that this exponential dependence
  of the regret on the horizon length $\H$ is necessary: we give
  examples of MDPs where no learning algorithm can achieve regret
  better than $A^{\nicefrac{\H}2}\cdot \sqrt{T}$, even when $k=2$.
\item \emph{(Ordered MDPs under Bandit Feedback)} We then obtain better regret bounds for the
  special class of \emph{ordered} MDPs, where there is an order on the
  states at each stage, and every transition is monotonically
  non-increasing in the states. We show how to refine the exploration
  step of our main algorithm to obtain an improved regret of roughly
  $(\H A)^k\sqrt{T \log T}$. 
\item \emph{(Lower Bound)} Again, we show that the exponential
  dependence on $k$ is necessary: no learning algorithm can achieve
  regret better than $(\H A)^{\nicefrac{k}2}\cdot \sqrt{T}$, even for
  ordered MDPs.
  
\item \emph{(Applications)} We show that ordered MDPs capture a number of
  classical stochastic optimization problems. For instance, we obtain
  the first bandit learning algorithms for the widely studied $k$-item
  prophet inequality.

\item \emph{(Experiments)} Finally, we evaluate the empirical performance of
  our learning algorithm for $k$-item prophet inequality, and find
  that our performance is comparable to the UCB-VI algorithm despite
  relying on significantly less feedback.

\end{enumerate}

\subsection{Related Work}\label{sec:related work}

The theoretical foundations of MDP learning are well-established, particularly within the episodic regime~\citep{JakOrAu2010, AzarOM17,JinABJ18,ZanetteB19,lattimore2020bandit}. A standard episodic MDP is characterized by a finite state space of size $S$, an action space  of size $A$, and a horizon $H$, with total episode rewards normalized to $[0,1]$. In our framework, the states correspond to levels ($S=k$). 
In the classical \emph{semi-bandit feedback} model, the agent receives high-resolution information, observing every visited state-action pair and its  instantaneous reward. In this setting, the minimax regret is tightly characterized by  
$\widetilde{\Theta}(\sqrt{\H kAT})$~\citep{AzarOM17,JakOrAu2010}. Another result by \citet{JinABJ18} provides a  $\TO(\H \sqrt{kAT})$ regret  under semi-bandit feedback where the algorithm is restricted to learn the value function directly.

In response to concerns about the inability/cost to observe per-step
rewards, or to apportion the reward across the various steps, there
has been research on MDP learning under limited feedback
structures. \citet{EfroniMM21} investigated \emph{trajectory
  feedback}, where the agent observes the realized state-action
trajectory but only the aggregate reward for the
episode.  
They established a regret bound of $\TO(\H kA\sqrt{T})$ via a
computationally intractable algorithm, and of
$\TO(\poly( A,k, \H)\sqrt{T})$ via computationally efficient ones.
\citet{CohenKKM21} studied MDP learning under trajectory feedback in
the adversarial setting, which is more challenging than the stochastic
case of \citet{EfroniMM21}. \citet{ChatterjiPBJ21} also consider
trajectory feedback in MDPs, where the feedback is binary according to a logistic model that depends on the trajectory.

Further extending the trajectory feedback research,
\citet{TangC00LS25, DuWDM025} studied \emph{segment feedback}, where
the trajectory is partitioned into $m$ equal segments. In addition to
the entire trajectory, the agent learns, at the end of each segment,
either (i) the sum of rewards in that segment (\emph{sum segment
  feedback}) or (ii) a Bernoulli random variable whose mean equals the
sigmoid of the segment's total reward (\emph{binary segment
  feedback}). For sum segment feedback, they establish upper and lower
bounds of $\widetilde{O}(\poly(k, A)\sqrt{\H T})$ and 
$\Omega(\sqrt{A k \H T})$. For binary segment feedback, the upper and
lower bounds become
$\widetilde{O}(\exp(\nicefrac{\H}{m})\poly(k, A) \sqrt{T})$ and
$\Omega(\exp(\nicefrac{\H}{m})\sqrt{A kT})$.

Our work is closely related to results on bandit learning algorithms
for \emph{structured MDPs} induced by stochastic optimization
problems, such as \emph{(single-item) prophet inequalities} and
\emph{posted pricing mechanisms}. These problems involve $\H$
independent random variables, each with support size at most
$A$. (Some results extend to continuous distributions.) At each time,
the learner proposes a policy, and learns only the reward accrued by a run of the policy on the associated MDP.   \citet{GatmiryKSW24} give
bandit feedback algorithms for the $1$-item prophet inequality and Pandora's
box problems with regret bound of $O(\H^3 \sqrt{T} \log T)$. (This result also holds for continuous distributions.) \citet{SinglaWang24a} study the
\emph{sequential posted price} problem and give an algorithm with
regret $O(\H^{2.5} \sqrt{A T} \log T)$. Our results provide regret bounds of $O(\H^{2k+1} A^{k} \sqrt{ T \log T})$ for
$k$-item prophet inequality and sequential posted pricing, for any  $k$. 

At a
high level, our approach extends that of \citet{GatmiryKSW24} for the
single-item prophet inequality, which is a special case of ordered
MDPs: we generalize the idea of successively refining ``thresholds''
to significantly more general settings.  Moreover, we prove that our
regret bounds are nearly optimal  for this class of MDPs.

There is also recent work modeling stochastic combinatorial problems as
MDPs; these problems typically have an exponential state space, making
them intractable for traditional algorithms. Research on these
problems has focused on exploiting the internal problem structure such
as product distributions and (weak) monotonicity to circumvent this
exponential dependence, albeit using richer semi-bandit feedback
(where we learn the trajectory and the per-state rewards). For example, see 
\citet{Guo0T021,AgarwalGN24,agarwal2025semibanditlearningmonotonestochastic,Liu-et-al-2025}.

\section{Problem Formulation}
\label{subsec:General MDP}

We consider an episodic Markov Decision Process (MDP), given by the
tuple \((\H, k, A, P, r).\) Here, $\H$ is the number of stages (horizon
of each episode), $k$ is the ``width'' (i.e., number of states) in
each stage, and $A$ is the set of actions. For each stage $i\in [\H]$
and level $l\in [k]$, there is a corresponding state $(l,i)$. The
starting state is in stage $1$, say $(1,1)$.  The transition kernel is
given by $P$, and transitions are allowed only between consecutive
stages. In particular, the transition probability from state $(l,i)$
to state $(s,i+1)$ under action $a$ is denoted by $p_i(s \mid l, a).$
The expected reward at any state $(l,i)$ under action $a$ is denoted
by $r_{l,i}(a)$. We assume that all rewards are non-negative and the
total reward in any execution of the MDP is bounded by
one.

Any deterministic policy can be represented as a $k \times \H$ matrix
$\Pi=(\pi_{l,i})_{l\in [k],i\in [\H], }$, where $\pi_{l,i}$ is the
action taken at state $(l,i)$. We also use
$\bm{\pi}_i=[\pi_{1,i},\cdots \pi_{k,i}]$ to denote the vector of
actions in stage $i$.\footnote{We use boldface letters (e.g., $\ba$)
  to denote vectors and plain letters (e.g., $a$) to denote scalars.}
The trajectory followed under policy $\Pi$ is the sequence
$\{(l_i, i)\}_{i=1}^{\H}$ of states that it encounters; we have $l_1=1$
for the start state. The action performed at state $(l_i,i)$ is
$\pi_{l_i,i}$, and so, conditional on reaching level $l_i$ in stage
$i$, the next state $l_{i+1}$ in the next stage is drawn from the
probability distribution $p_i(\cdot | l_i, \pi_{l_i,i})$. The reward
under such a trajectory is $ \sum_{i=1}^{\H} R_{l_i,i}(\pi_{l_i,i})$,
where $R_{l,i}(a)$ is the random reward at state $(l,i)$ under action
$a$. The value of policy $\Pi$ is the expected total reward
$ \E \big[ \sum_{i=1}^{\H} r_{l_i,i}(\pi_{l_i,i})\big] $, where the
expectation is taken over the trajectory. We are interested in finding
the policy of maximum value. A randomized policy can be viewed as a
convex combination of deterministic policies; note that there is
always an optimal deterministic policy.

We will also work with policies that start from some intermediate
stage $i$. Formally, a \emph{tail policy in stage $i$} is given by a
$k\times (\H-i+1)$ matrix $[\ba_i, \dots, \ba_{\H}]$, and the
\emph{state-action value function} at any state $(l,i)$ is:
\begin{equation}\label{eq:value function}
    V_{l,i}(\ba_i, \ba_{i+1}, \dots, \ba_{\H}) 
= \E \bigg[ \sum_{p=i}^{\H} r_{l_p,p}(a_{l_p,p}) \,\Big|\, l_i = l \bigg].
\end{equation}
where $\{l_p,p\}_{p=i}^{\H}$  are the states  encountered during the rest
of the episode; we have $l_i=l$ in stage $i$. 
We have the standard decomposition:
\begin{align}\label{eq:bellman op}
  & V_{l,i}(\ba_i, \ba_{i+1}, \dots, \ba_{\H}) = \notag \\ & r_{l,i}(a_{l,i}) + \sum_{s \in [k]} p_i(s | l, a_{l,i}) \cdot V_{s,i+1}(\ba_{i+1}, \dots, \ba_{\H}).
\end{align}
We say
$V_{l,i}(\ba_i, \ba_{i+1}, \dots, \ba_{\H})=V_{l,i}(a, \ba_{i+1}, \dots,
\ba_{\H})$ where $a_{l,i}=a$ 
is the action taken at state $(l, i)$, since the value function
$V_{l,i}$ does not depend on all actions in stage $i$, but only on the
action at state $(l,i)$. Finally, $V(\pi)$ denotes the expected reward
of policy $\pi$, which is the value function at the start state.

\paragraph{Bandit learning setting.} 
We are interested in  learning an optimal policy when the transition kernel and rewards are {\em unknown}.  
We assume that $\H$, $k$ and $A$  are known.  Specifically, we consider the online learning setting where an agent interacts with the MDP over $T$ episodes. 
In each episode $t\in [T]$, the agent chooses a policy $\pi^t$ and receives as feedback just the cumulative reward from one policy execution of $\pi^t$. This is the setting of  
{\em  bandit} feedback. The goal is to minimize the cumulative regret over $T$ episodes: 
$$
\mathcal{R} (T) = \sum_{t=1}^T \opt - V(\pi^t) = T\cdot \opt - \sum_{t=1}^T  V(\pi^t),
$$
where $\opt$ is the expected reward of the optimal policy.

In  the bandit feedback setting that we consider, if the agent chooses policy $\pi$ in some  episode then it  
only observes the  cumulative reward  $ \sum_{i=1}^{\H} r_{l_i,i}(\pi_{l_i,i})$, where $\{(l_i, i)\}_{i=1}^{\H}$  denote the states   encountered during the episode. In particular, the agent does not observe  the sequence of states  visited by policy $\pi$. This makes our setting significantly harder than the 
 standard \emph{semi-bandit feedback} setting.

\section{Algorithm for MDP with Bandit Feedback}\label{sec:Bandit Algorithms}

Our first main result is:

\begin{theorem}\label{thm:main-mdp}
  There is an online learning algorithm for MDPs with unknown
  transition probabilities and bandit feedback having regret
  $O\left(\H^2(Ak)^{\H}\sqrt{\H kA T \log T }\right)$. 

\end{theorem}

A natural first attempt for MDP learning under bandit feedback is to
apply the UCB algorithm by treating each policy as a single
arm. However, this approach leads to a regret bound of
$A^{\nicefrac{\H k}2}\cdot
\sqrt{T}$ 
as the total number of policies is $A^{\H k}$. This approach is also
computationally inefficient as it requires storing an index for every
policy. Our approach obtains a significantly better regret bound of
roughly $(Ak)^{\H } \sqrt{T}$, which we prove is nearly optimal.

Our algorithm is inspired by the approach of~\citet{GatmiryKSW24} for single-item prophet inequality, which can be seen as a sophisticated  successive-elimination algorithm (though it is not stated as such). We provide a much more general successive-elimination framework for arbitrary episodic MDPs as well as the special  class of ``ordered'' MDPs.  One of our main contributions is in coming up with the correct elimination thresholds (as a function of the stage and level) that enables an inductive proof of the regret bound. Our analysis is modular in the sense that, given any set of thresholds satisfying a technical condition (see Lemma~\ref{lem:unif-induction}) we prove a corresponding regret bound. This leads to a unified  analysis for the two cases (arbitrary and ordered MDPs), which only differ in the choice of   thresholds (resulting in different regret bounds). We also show these regret bounds are near-optimal for both cases.

 Our algorithm proceeds in phases indexed by an error parameter
$\epsilon>0$, maintaining
an ``active set'' $A_{l,i}$ of actions for each state $(l,i)$. The precondition in each phase is that any policy using actions in $\{A_{l,i}\}$ has one-step regret 
$\approx (Ak)^{\H}\cdot \epsilon$. Then, the algorithm  performs $O(\frac1{\epsilon^2} Ak\H \log T)$ episodes with various policies using the  active sets  $\{A_{l,i}\}$.  It then refines its active
sets based on bandit feedback to obtain a new collection $\{N_{l,i}\}$.  Crucially, we show that every policy
using the refined active sets has one-step regret  $\approx (Ak)^{\H}\cdot \frac\epsilon{2}$
optimal. This ensures that we make significant progress towards optimality: the error
parameter $\epsilon$ reduces by factor two in each phase.

\subsection{Explore then Refine}\label{subsec expref}

In this subsection, we introduce   the \emph{explore-then-refine} algorithm ($\Alg$) which corresponds to a single phase in the overall algorithm. 

This algorithm is divided into two components: \emph{sampling actions} and \emph{learning}, which correspond respectively to transition control and reward estimation.  

\textbf{Sampling Actions}   The probability of reaching a particular state $(l,i)$ depends on the sequence of actions taken before reaching it. This probability is difficult to control under bandit feedback. If it is too small, we cannot obtain sufficient information about that state. To overcome this issue, $\Alg$ uniformly samples actions at all states in the  stages before $i$. This guarantees that the probability of reaching any given state is at least some fraction of the corresponding probability in  the optimal policy. A formal statement of this property is provided in Lemma~\ref{lem: unif sampling}.  

\textbf{Learning} Another challenge is that the value function~\eqref{eq:value function} of a state also depends on future actions. To identify the optimal action $a_{l,i}^*$ at state $(l,i)$, one would also need to know the corresponding optimal tail policy. Hence, instead of attempting to directly learn the optimal action, $\Alg$ proceeds via backward induction. Specifically, starting at $i=\H$, the algorithm performs uniform sampling over all preceding actions and explores every possible action for states $(\cdot,\H )$ in stage $\H$. It then selects the action $\halpha_{l,\H}$ associated with the largest empirical reward for each state $(l,\H)$. Moving backward to stage $i=\H-1, \dots, 1$, for each state $(\cdot,i)$ the algorithm fixes the tail policy to be $\bhalpha_{j \geq i}$, where $\bhalpha_{j} = (\halpha_{k,j}, \dots, \halpha_{1,j})$. $\Alg$ will explore all possible actions for states $(\cdot,i)$ and then chooses the empirically best action $\halpha_{\cdot,i}$ under this fixed tail policy.  Moreover, we refine the active set for all states in
stage $i$ by eliminating actions that lead to policies which are inferior to the empirical best policy by  some threshold $\approx (Ak)^{\H -i}\cdot \epsilon$.

At the end of this process, the algorithm produces an empirically best policy $\bhalpha$ as well as a refined action set $N_{l,i}$ at every state $(l,i)$. 
As long as the optimal policy is contained in the initial action-sets  $\{A_{l,i}\}$, we will show that any policy chosen from the refined action-sets  $\{N_{l,i}\}$ has regret $\propto \epsilon$.  See Lemma~\ref{lem:shrink key lemma} for the formal statement.

\begin{algorithm}
\caption{Explore then Refine Algorithm (\Alg)}
\label{alg:unifexp}
\begin{algorithmic}[h]
  \State \textbf{Input:} action set $A_{l,i} \text{ for all } (l,i) \in [k] \times [\H]$, and $\epsilon\in (0,1]$
\For{$(l,i) \in [k] \times [\H]$} \Comment{\textbf{Sampling Actions}}
  \State Sample $\e_{l,i} $ uniformly from $A_{l,i}$ independently. 
\EndFor
  \State Let $\mathbf{\e_i} = (\e_{k,i} \dots \e_{1,i})  $ be the vector of ``explore'' actions in stage $i$.
  \For{$i = \H, \H-1, \dots 1$} \Comment{\textbf{Learning}}
    \For{$l=k, \dots 1$}
    \For{ $a\in A_{l,i}$ }
    \State Let $\ba =  (\e_{k,i}, \dots  \e_{l+1,i},a, \e_{l-1,i}\dots \e_{1,i} )$.
    \State Play policy \([\mathbf{e_1}, \dots, \mathbf{e_{i-1}}, \mathbf{a},\bhalpha_{i+1}, \dots, \bhalpha_{\H}] \) for  \(\frac{12\log T}{\epsilon^2}\) episodes   
    \State Let  \(\hat{\Phi}_{l,i}(a)\) be the empirical average reward for this policy.
    \EndFor
    \State Let  $\halpha_{l,i} = \arg\max\limits_{a\in A_{l,i}} \, \hat{\Phi}_{l,i}(a)$ be the empirical best action at state $(l,i)$.
    \State Let the refined action set be
$$ N_{l,i} =  \left\{a \in A_{l,i}  \Bigm|  
               \hat{\Phi}_{l,i}(a)  \ge  \hat{\Phi}_{l,i} (\halpha_{l,i}  ) -  C_{l,i}\,\cdot \epsilon
              \right\},\,\, \mbox{ where   $C_{l,i}$ is defined in \eqref{eq:C_l,i}.}$$ 
              \EndFor
    \State     Let  $\bhalpha_{i} = (\halpha_{k,i} \dots \halpha_{1,i})$ be the vector of empirical best actions in stage $i$.
  \EndFor
  
  \State \textbf{Output}  refined action set $ N_{l,i}$ for all $(l,i) \in [k] \times [\H]$.
\end{algorithmic}
\end{algorithm}

\paragraph{Random Exploration Policy $\eta$.} Note that the algorithm samples a random (active) action $e_{l,i}\in A_{l,i}$ at each state. Let $\eta=(e_{l,i})_{l\in [k], i\in [\H]}$ denote the randomized policy that uses these actions. For any $l$ and $i$, we define the ``visitation probability'' $Q_{l,i}$  as the probability that  policy $\eta$ visits  state   $(l,i)$. 

\begin{lemma}\label{lem: unif sampling}
For any $i \in [\H]$ and $l,s \in [k]$, we have
\[
Q_{s,i+1}  \geq  \max_{a \in A_{l,i}}\frac{Q_{l,i}}{A}\, p_i(s \mid l,a).
\]
\end{lemma}

\begin{proof} Fix any action $a\in A_{l,i}$. 
In order to reach state $(s,i+1)$, one possible path is to first reach state $(l,i)$, then take action $a$ at $(l,i)$, and transition to $(s,i+1)$. 
\begin{align*}
    Q_{s,i+1} &= \pr[ \eta \text{ reaches} (s,i+1)] \\
        &\geq \pr\left[ \eta \text{ reaches } (l,i)\, \bigwedge\, \text{sampled action }e_{l,i}= a \,\bigwedge\,
      \text{action $a$ at $(l,i)$  transitions to } (s,i+1)\right] \\
    &= \pr\left[ \eta \text{ reaches } (l,i)\right]
        \cdot 
        \pr\left[\text{sampled action }e_{l,i}= a\right] \cdot \pr[\text{action $a$ at $(l,i)$  transitions to } (s,i+1)] \\
   & =  Q_{l,i}\cdot \frac1{|A_{l,i}|}\cdot p_i(s \mid l,a)  \quad  \geq \quad   \frac{Q_{l,i}}{A}\, p_i(s \mid l,a).
\end{align*}

The second equality uses the fact that the action chosen at $(l,i)$ is independent of all other selections and transitions.
The last equality follows from the fact that $e_{l,i}$ is sampled uniformly from  $A_{l,i}$. 
Taking the maximum over all $a\in A_{l,i}$ completes the proof. 
\end{proof}

We now introduce a sequence of constants that determine how we refine the action sets in each stage.  For any $(l,i) \in [k] \times [\H]$, define
\begin{equation}\label{eq:C_l,i}
C_{l,i}= (\H-i+1)(Ak)^{\H -i}.    
\end{equation}

Then, we have the following inductive relationship.
\begin{lemma}\label{lem:unif-induction}
    For any $(l,i) \in [k] \times [\H]$, and $a \in A_{l,i}$ we have:
    $$ 
    \frac{\epsilon}{Q_{l,i}}       +   \sum_{s \in [k]} p_i(s \mid l, a)   \frac{C_{s,i+1} }{Q_{s,i+1} } \cdot   \epsilon  \leq \frac{C_{l,i}  }{Q_{l,i}}\cdot   \epsilon .
    $$
\end{lemma}

\begin{proof} This follows by the choice of the constants in \eqref{eq:C_l,i} and Lemma~\ref{lem: unif sampling}. 
    \begin{align*}
        &\frac{1}{Q_{l,i}}       +   \sum_{s \in [k]} p_i(s \mid l, a)   \frac{C_{s,i+1}  }{Q_{s,i+1} } \quad = \quad   \frac{1}{Q_{l,i}}     +\sum_{s \in [k]} p_i(s | l, a) \cdot \frac{(\H-i)(Ak)^{\H -i-1} }{Q_{s,i+1}}\\
        & \leq \frac{1}{Q_{l,i}}+ \sum_{s \in [k]} A \cdot \frac{(\H-i)(Ak)^{\H -i-1} }{Q_{l,i}}  \qquad  \text{(by Lemma~\ref{lem: unif sampling})} \\
        & \leq \frac{1}{Q_{l,i}} +  \frac{(\H-i)(Ak)^{\H -i} }{Q_{l,i}}\leq  \frac{(\H-i+1)(Ak)^{\H -i} }{Q_{l,i}}=  \frac{C_{l,i} }{Q_{l,i}}.
    \end{align*}
\end{proof}

Henceforth, we will not use the exact definition of the constants~\eqref{eq:C_l,i}: we will only use Lemma~\ref{lem:unif-induction} and the facts  $C_{\star,\H}=1$ and $C_{1,1}=\H(Ak)^{\H -1}$.

\subsection{Regret Analysis}

We provide a regret analysis for the subroutine $\Alg$ in this section. Throughout this section, we use $\ba^*$ to denote the optimal policy, and $\opt$ to denote the expected reward of $\ba^*$.

\begin{lemma}\label{lem:shrink key lemma}
Suppose that the optimal action $a_{l,i}^* \in A_{l,i}$ for all $ (l,i) \in [k] \times [\H]$. Then, given any $\epsilon>0$,

with probability at least $1 - T^{-2}$, we have the following:
\begin{enumerate}
    \item For all $(l,i) \in [k] \times [\H]$, the optimal action $a_{l,i}^*\in N_{l,i}$. 
    \item For any policy $\Pi \in  N_{1,1} \times \dots \times N_{k ,\H}$, the expected reward  of $\Pi$ is at least $\opt - 3\H C_{1,1} \epsilon$.
\end{enumerate}
\end{lemma}

Before proceeding with the proof, we introduce some notations.  For each action $a \in A_{l,i}$, we use $\Phi_{l,i}(a)$ to denote  the true expected reward of the policy 
$$
    [\mathbf{e}_1, \dots, \mathbf{e}_{i-1}, \ba, \bhalpha_{i+1}, \dots, \bhalpha_{\H}],
$$
   where $\ba =   (e_{k,i}, \dots  e_{l+1,i},a, e_{l-1,i},\dots e_{1,i} )$. We further define the ``best'' actions as:
       \[
    \alpha_{l,i} = \arg\max_{a \in A_{l,i}} \Phi_{l,i}(a),\quad \forall l \in [k] , \, \forall i\in [\H].
    \]
Recall from  $\Alg$ that  $\hat{\Phi}_{l,i}(a)$ represents the  empirical reward of policy  $
    [\mathbf{e}_1, \dots, \mathbf{e}_{i-1}, \ba, \bhalpha_{i+1}, \dots, \bhalpha_{\H}]$ 
and $\bhalpha_{l,i}$  is the empirical  best action at state $(l,i)$.

By the above definitions, for any state $(l,i)$, we obtain a useful relationship among $\Phi_{l,i}(\cdot)$, the value function $V_{l,i}(\cdot)$, and the visitation probability $Q_{l,i}$. Specifically, the difference in $\Phi_{l,i}(\cdot)$ is equal to the difference in the value function multiplied by the visitation probability of the state.

In the following, we use the shorthand $\bhalpha_{j > i}=(\bhalpha_{i+1},\cdots \bhalpha_{\H})$ and $\bhalpha_{j \ge i}=(\bhalpha_{i },\cdots \bhalpha_{\H})$. 

\begin{lemma}\label{obs: to-go-cost decomposition}
For any $i\in [\H]$, $l\in [k]$ and actions $a, b \in A$, 
\begin{equation}\label{eq:Phi and V}
\Phi_{l,i}(a) - \Phi_{l,i}(b) = Q_{l,i} \cdot \left[V_{l,i}(a, \bhalpha_{j > i}) - V_{l,i}(b, \bhalpha_{j > i})\right].    
\end{equation}
Therefore, we obtain:
\begin{equation}\label{eq:opt alpha}
    \alpha_{l,i} = \arg\max_{a \in A_{l,i}} V_{l,i}(a, \bhalpha_{j > i}).
\end{equation}
\end{lemma}
\begin{proof}
    We are interested in the difference in objective of the two policies:
$$  \alpha=  [\mathbf{e}_1, \dots, \mathbf{e}_{i-1}, \ba, \bhalpha_{i+1}, \dots, \bhalpha_{\H}] \quad \mbox{and}\quad 
\beta=  [\mathbf{e}_1, \dots, \mathbf{e}_{i-1}, \mathbf{b}, \bhalpha_{i+1}, \dots, \bhalpha_{\H}],$$
where  $\ba =   (e_{k,i}, \dots  e_{l+1,i},a, e_{l-1,i},\dots e_{1,i} )$ and  $\mathbf{b} =   (e_{k,i}, \dots  e_{l+1,i}, b, e_{l-1,i},\dots e_{1,i} )$. Note that policies $\alpha$ and $\beta$ only differ in the action at state $(l,i)$. So, their objectives only differ if state $(l,i)$ is visited. Moreover, the actions in stages $1,2,\cdots i-1$ of policies $\alpha$ and $\beta$ are the same as the random exploration policy $\eta$.  So, the probability of visiting state $(l,i)$ is exactly $Q_{l,i}$. Furthermore, the expected reward of policy $\alpha$ from stages $i,i+1,\cdots n$ {\em conditioned} on visiting   state $(l,i)$ is $V_{l,i}(a,\bhalpha_{i+1},\cdots \bhalpha_{\H})=V_{l,i}(a,\bhalpha_{j > i})$. Similarly, for policy $\beta$ the conditional  expected reward from stages $i,i+1,\cdots n$ is $V_{l,i}(a,\bhalpha_{j > i})$.  So,
$$\Phi_{l,i}(a) - \Phi_{l,i}(b) = \Pr[\eta\mbox{ reaches }(l,i)] 
\left( V_{l,i}(a,\bhalpha_{j > i}) - V_{l,i}(b,\bhalpha_{j > i})\right) = Q_{l,i}\cdot
\left( V_{l,i}(a,\bhalpha_{j > i}) - V_{l,i}(b,\bhalpha_{j > i})\right) $$
This proves \eqref{eq:Phi and V}. Using this, it follows that for any fixed $l$ and $i$, maximizing $\Phi_{l,i}(x)$ is equivalent to maximizing $V_{l,i}(x)$, which proves \eqref{eq:opt alpha}. 
\end{proof}

\begin{lemma}\label{lem: Hoeffding}
With probability at least $1 - T^{-2}$ we have
$$
     |\Phi_{l,i}(a) - \hat{\Phi}_{l,i}(a) | \leq \frac{\epsilon}{2}.
$$
for any $(l,i) \in [k] \times [\H]$, and $a \in A_{l,i}$.
\end{lemma}

\begin{proof}
For each $l,i,a$, $\hat{\Phi}_{l,i}(a)$ is an unbiased estimator of $\Phi_{l,i}(a)$ obtained from $ \frac{12 \log T}{\epsilon^2}$independent samples. We can now use Hoeffding’s inequality~\citep{Hoeffding1963} due to the regularity assumption on  rewards (namely, the total reward is always $[0,1]$ bounded). So, for any fixed $a \in A_{l,i}$,
\[
    \Pr\!\left(  |\hat{\Phi}_{l,i}(a) - \Phi_{l,i}(a) | > \frac{\epsilon}{2} \right) 
    \leq 2 \exp(-\frac{\epsilon^2}{2} \cdot  \frac{12 \log T}{\epsilon^2}) \leq 2T^{-6}.
\]
And by a union bound over all $(l,i,a)$ (assume $T$ is larger than $k$, $\H$, and $A$), this probability is at most $T^{-2}$.  
Hence, with probability at least $1 - T^{-2}$ the stated bound holds.
\end{proof}

In the following, we assume that the ``good event'' in Lemma~\ref{lem: Hoeffding} holds. Our first result shows that, for any state $(l,i)$, the empirical maximizer $\halpha_{l,i}$  has nearly the same value function as the true maximizer $\alpha_{l,i}$. 
\begin{corollary}\label{cor:unif Hoeffding}
    Under the good event, for any $(l,i) \in [k] \times [\H] $ we have:
$$
     |V_{l,i}(\alpha_{l,i},\bhalpha_{j>i}) - V_{l,i}(\halpha_{l,i},\bhalpha_{j>i}) | \leq \frac{\epsilon}{Q_{l,i}}.
$$
\end{corollary}
\begin{proof}
  First, by \eqref{eq:opt alpha}, we have: 
\[
    V_{l,i}(\halpha_{l,i}, \bhalpha_{j>i}) - V_{l,i}(\alpha_{l,i}, \bhalpha_{j>i}) 
   \leq 0.
\]

For the other direction, we have:
\begin{align*}
V_{l,i}(\alpha_{l,i}, \bhalpha_{j>i}) - V_{l,i}(\halpha_{l,i}, \bhalpha_{j>i})
&= \frac{1}{Q_{l,i}}  ( \Phi_{l,i}(\alpha_{l,i}) - \Phi_{l,i}(\halpha_{l,i})  ) \\
&\leq \frac{1}{Q_{l,i}}  ( \epsilon + \hat{\Phi}_{l,i}(\alpha_{l,i}) - \hat{\Phi}_{l,i}(\halpha_{l,i})  ) \leq \frac{ \epsilon}{Q_{l,i}}.
\end{align*}

The equality follows from \eqref{eq:Phi and V},   the first inequality follows from Lemma~\ref{lem: Hoeffding}, and the second inequality follows from the fact $\halpha_{l,i}$  corresponds to the highest empirical reward $\max_{a\in A_{l,i}} \hat{\Phi}_{l,i}(a)$ and that $\alpha_{l,i}\in A_{l,i}$.    
\end{proof}

Recall that $a^*$ is the optimal policy. 
For any state $(l,i)$, we  use $\opt_{l,i}$ to denote the expected  reward of the optimal policy in stages $i,\cdots n$ conditioned on reaching $(l,i)$. We have 
\begin{equation}\label{eq:opt-to-go}
\opt_{l,i}  =V_{l,i}(a^*_{l,i}, \ba^*_{i+1},\cdots \ba^*_{\H} )  = \max_{a\in A} \,\,\max_{\ba_{i+1},\cdots \ba_{\H} \in A^k} V_{l,i} (a,\ba_{i+1},\cdots \ba_{\H}).    
\end{equation}

The following lemma shows that for any state $(l,i)$, the value function of the empirical best action $\bhalpha_{l,i}$ is close to that of the optimal action $a^*_{l,i}$. The error term depends on the stage $i$ and the visitation probability $Q_{l,i}$. 
\begin{lemma}\label{lem:unif-halpha-opt}
For any $ (l,i) \in [k] \times [\H]$, the following bound holds:
\[
\opt_{l,i} \leq V_{l,i}(\bhalpha_{j \geq i}) + \frac{C_{l,i}}{Q_{l,i}}\cdot \epsilon.
\]
\end{lemma}

\begin{proof}
We prove the lemma by backward induction on $i$ from $\H$ down to $1$. 

\textbf{Base case ($i = \H $)}. We have for all $l\in [k]$,
\[
\opt_{l,\H} = V_{l,\H}(a^*_{l,\H}) \leq V_{l,\H}(\alpha_{l,\H}) \leq V_{l,\H}(\halpha_{l,\H}) + \frac{\epsilon}{Q_{l,\H}} = V_{l,\H}(\halpha_{l,\H}) + \frac{C_{l,\H}}{Q_{l,\H}}\cdot \epsilon .
\]
The first inequality follows from \eqref{eq:opt alpha} and the assumption that $a^{*}_{l,\H } \in A_{l,\H}$. The second inequality is from Corollary~\ref{cor:unif Hoeffding}. The last equality uses $C_{l,\H}=1$. 

\textbf{Inductive step}. Assume the lemma holds for stage $i+1$. For any $l \in [k]$, we have:
\allowdisplaybreaks
\begin{align}
\opt_{l,i} &= r_{l,i}(a^*_{l,i}) + \sum_{s \in [k]} p_i(s \mid l, a^*_{l,i}) \opt_{s,i+1} && \text{ (Bellman Equation~\eqref{eq:bellman op})} \nonumber\\
&\leq r_{l,i}(a^*_{l,i}) + \sum_{s \in [k]} p_i(s \mid l, a^*_{l,i}) \left( V_{s,i+1}(\bhalpha_{j > i}) + \frac{C_{s,i+1} }{Q_{s,i+1} } \cdot \epsilon\right) && \text{ (Inductive Hypothesis)}   \nonumber\\
&= V_{l,i}(a^*_{l,i}, \bhalpha_{j > i}) + \sum_{s \in [k]} p_i(s \mid l, a^*_{l,i})   \frac{C_{s,i+1} }{Q_{s,i+1} } \cdot \epsilon \label{eq: Unif: OPT and V}\\
&\leq V_{l,i}(\alpha_{l,i}, \bhalpha_{j > i}) +  \sum_{s \in [k]} p_i(s \mid l, a^*_{l,i})   \frac{C_{s,i+1} }{Q_{s,i+1}} \cdot \epsilon &&   \text{($a^*_{l,i}\in A_{l,i}$ and \eqref{eq:opt alpha})}   \nonumber\\
&\leq V_{l,i}(\halpha_{l,i}, \bhalpha_{j > i})  + \frac{\epsilon}{Q_{l,i}}       +   \sum_{s \in [k]} p_i(s \mid l, a^*_{l,i})   \frac{C_{s,i+1} }{Q_{s,i+1} } \cdot \epsilon&& \text{(Corollary~\ref{cor:unif Hoeffding})}  \nonumber\\
& \leq  V_{l,i}(\halpha_{l,i}, \bhalpha_{j > i})  + \frac{C_{l,i} }{Q_{l,i}} \cdot \epsilon && \text{($a^*_{l,i}\in A_{l,i}$  and Lemma~\ref{lem:unif-induction})}. \nonumber
\end{align}

 This completes the proof. 
\end{proof}

We have established that the error in the value function is bounded and does not deviate significantly from the optimal value. The next lemma shows that by  maintaining the set of actions whose value functions lie within an appropriate neighborhood of the empirical maximizer, 
we can guarantee that the optimal action remains in the refined set. 

\begin{lemma}\label{lem: unif part a}
For all $(l,i) \in [k] \times [\H]$, we have  
$ a_{l,i}^* \in N_{l,i}.$
\end{lemma}
\begin{proof}
By definition of the refined set $\H_{l,i}$, proving  $a_{l,i}^* \in N_{l,i}$ is equivalent to showing:
\[
    \hat{\Phi}_{l,i}(a_{l,i}^*)  \geq  \hat{\Phi}_{l,i}(\halpha_{l,i}) - C_{l,i}\cdot \epsilon.
\]

Indeed, we have
\begin{align*}
\hat{\Phi}_{l,i}(\halpha_{l,i}) - \hat{\Phi}_{l,i}(a_{l,i}^*) 
&\leq \Phi_{l,i}(\halpha_{l,i}) - \Phi_{l,i}(a_{l,i}^*) + \epsilon &&\text{(by Lemma~\ref{lem: Hoeffding})} \\
&= Q_{l,i}  ( V_{l,i}(\halpha_{l,i}, \bhalpha_{j > i}) - V_{l,i}(a_{l,i}^*, \bhalpha_{j > i})  ) + \epsilon  &&\text{(by \eqref{eq:Phi and V})} \\
&\leq Q_{l,i} ( \opt_{l,i} - V_{l,i}(a_{l,i}^*, \bhalpha_{j > i})  ) + \epsilon  \\
&\leq Q_{l,i} \left (\sum_{s \in [k]} p_i(s \mid l, a^*_{l,i})  \frac{C_{s,i+1}  }{Q_{s,i+1} } \cdot \epsilon+ \frac{\epsilon}{Q_{l,i}}  \right)&& \text{(by \eqref{eq: Unif: OPT and V})}\\
&\leq C_{l,i}\cdot  \epsilon && \text{($a^*_{l,i}\in A_{l,i}$  and Lemma~\ref{lem:unif-induction})}.
\end{align*}

\end{proof}

The next lemma shows that any policy using actions  from the refined  sets $\{\H _{l,i}\}$ (defined in Algorithm~\ref{alg:sl}) has reward  comparable to  the empirical best policy $\bhalpha$.

\begin{lemma}\label{lem:unif-final-regret}
Consider any state $(l,i) \in [k] \times [\H]$. For any  actions $\pi_{s,j}\in N_{s,j}$ for $s\in [k]$ and $j\in \{i,\cdots n\}$,  we have:
\[
V_{l,i}(\halpha_{l,i}, \bhalpha_{j > i}) - V_{l,i}(\bm{\pi}_{i}, \dots, \bm{\pi}_{\H }) \leq \frac{2(\H-i+1)C_{l,i} }{Q_{l,i}}\cdot \epsilon.
\]
\end{lemma}
\begin{proof}
At a high level, this proof is   similar to Lemma~\ref{lem:unif-halpha-opt}. We again proceed by backward induction. 

\textbf{Base case ($i = \H $)}. For any $l\in [k]$, we have
\begin{align*}
V_{l,\H}(\halpha_{l,\H}) - V_{l,\H }(\pi_{l,\H})
&= \frac{\Phi_{l,\H}(\halpha_{l,\H}) - \Phi_{l,\H}(\pi_{l,\H})}{Q_{l,\H}} \\
&\leq \frac{\hat{\Phi}_{l,\H}(\halpha_{l,\H}) - \hat{\Phi}_{l,\H}(\pi_{l,\H}) + \epsilon}{Q_{l,\H}} \,\,\leq\,\, (C_{l,\H}+1)\cdot \frac{\epsilon}{Q_{l,\H}} \leq \frac{2C_{l,\H}  }{Q_{l,\H }}\cdot \epsilon.
\end{align*}
The first inequality is by the good event (Lemma~\ref{lem: Hoeffding}), the second inequality uses  $\pi_{l,\H}\in N_{l,\H}$ and the last inequality is by  $C_{l,\H}=1$.   

\textbf{Inductive step.} Suppose the claim holds for stage $i+1$. Then we have:
\allowdisplaybreaks
\begin{align*}
& V_{l,i}(\halpha_{l,i}, \bhalpha_{j > i}) - V_{l,i}(\bm{\pi}_{i}, \dots, \bm{\pi}_{\H}) \\
&=V_{l,i}(\halpha_{l,i}, \bhalpha_{j > i})  - r_{l,i}(\pi_{l,i}) - \sum_{s \in [k]} p_i(s | l, \pi_{l,i}) V_{s,i+1}(\bm{\pi}_{i+1}, \dots, \bm{\pi}_{\H}) \quad \text{(By \eqref{eq:bellman op})} \\
&\leq V_{l,i}(\halpha_{l,i}, \bhalpha_{j > i})  - r_{l,i}(\pi_{l,i})- \sum_{s \in [k]} p_i(s | l, \pi_{l,i}) \left( V_{s,i+1}(\bhalpha_{j > i}) - \frac{2(\H-i)C_{s,i+1}  }{Q_{s,i+1}} \cdot \epsilon\right)  \quad \text{(Induction)}\\
&= V_{l,i}(\halpha_{l,i}, \bhalpha_{j > i})-V_{l,i}(\pi_{l,i}, \bhalpha_{j > i}) + \sum_{s \in [k]} p_i(s | l, \pi_{l,i}) \frac{2(\H-i)C_{s,i+1} }{Q_{s,i+1}} \cdot \epsilon\\
&= \frac{\Phi_{l,i}(\halpha_{l,i}) - \Phi_{l,i}(\pi_{l,i})}{Q_{l,i}} +
\sum_{s \in [k]} p_i(s | l, \pi_{l,i}) \frac{2(\H-i)C_{s,i+1}  }{Q_{s,i+1}} \cdot \epsilon \qquad  \text{(by Lemma~\ref{obs: to-go-cost decomposition})}\\
&\leq \frac{\hat{\Phi}_{l,i}(\halpha_{l,i}) - \hat{\Phi}_{l,i}(\pi_{l,i})}{Q_{l,i}} 
+ \frac{\epsilon}{Q_{l,i}} 
+ \sum_{s \in [k]} p_i(s | l, \pi_{l,i}) \frac{2(\H-i)C_{s,i+1} }{Q_{s,i+1}}\cdot \epsilon  \qquad  \text{(by Lemma~\ref{lem: Hoeffding})}\\
& \leq  \frac{\hat{\Phi}_{l,i}(\halpha_{l,i}) - \hat{\Phi}_{l,i}(\pi_{l,i})}{Q_{l,i}} 
+ 2(\H-i)  \epsilon \cdot  \left(\frac{1}{Q_{l,i}}  +  \sum_{s \in [k]} p_i(s | l, \pi_{l,i}) \frac{C_{s,i+1} }{Q_{s,i+1}} \right) \qquad  \text{(by  $i\le \H-1$)}\\
& \leq  \frac{\hat{\Phi}_{l,i}(\halpha_{l,i}) - \hat{\Phi}_{l,i}(\pi_{l,i})}{Q_{l,i}} 
+ \frac{2(\H-i)C_{l,i}  }{Q_{l,i}}\cdot \epsilon \qquad  \text{(by Lemma~\ref{lem:unif-induction})}\\
&\leq  \frac{C_{l,i}  }{Q_{l,i}}\cdot \epsilon+ \frac{\epsilon}{Q_{l,i}} 
+ \frac{2(\H-i)C_{l,i}  }{Q_{l,i}}\cdot \epsilon \qquad \text{(by $\pi_{l,i}\in N_{l,i}$)}\\
&\leq  \frac{2(\H-i+1) C_{l,i} }{ Q_{l,i}}\cdot \epsilon. 
\end{align*}
\end{proof}

\paragraph{Completing the proof of Lemma~\ref{lem:shrink key lemma}.} The first statement in the lemma (that the optimal action is contained in the refined set) follows directly from Lemma~\ref{lem: unif part a}. For the second statement, consider any policy using actions $\pi_{l,i}\in N_{l,i}$ for $l\in [k]$ and $i\in[\H]$. We have,   

\begin{align*}
\opt-V(\bm{\pi})&= \opt_{1,1} - V_{1,1}(\bm{\pi}_1, \dots, \bm{\pi}_{\H}) 
\\&=\left[ \opt_{1,1} - V_{1,1}(\bhalpha_{j \geq 1}) \right] + \left[ V_{1,1}(\bhalpha_{j \geq 1}) - V_{1,1}(\bm{\pi}_1, \dots, \bm{\pi}_{\H}) \right] \\
& \leq (1+2\H)C_{1,1} \epsilon \leq 3 \H C_{1,1} \epsilon,    
\end{align*}

where we used  Lemma~\ref{lem:unif-halpha-opt} and Lemma~\ref{lem:unif-final-regret} and the fact that $Q_{1,1}=1$.

\paragraph{Doubling trick for overall algorithm.} We now show that a simple doubling idea can by used on top of Lemma~\ref{lem:shrink key lemma} to obtain the upper bound in Theorem~\ref{thm:main-mdp}. The overall algorithm is given in Algorithm~\ref{alg:doubling}. 

\begin{algorithm}[H]
    \caption{Doubling Algorithm}
    \label{alg:doubling}
    \begin{algorithmic}[1]

        \State \textbf{Initialize:}  $\epsilon  \gets 1$, and action sets $A_{l,i}=A$ for all $l\in [k],i\in [\H]$.
        \While{$\epsilon > \sqrt{\frac{ \H kA \log T}{T }}$}
            \State Call $\Alg$ with  parameter $\epsilon$ and action sets $\{A_{l,i}\}$ to obtain refined action sets $\{\H _{l,i}\}$.
            \State update $\epsilon  \gets \epsilon / 2$ and action set $A_{l,i}\gets N_{l,i}$ for all $l\in [k],i\in [\H]$.
        \EndWhile
    \end{algorithmic}
\end{algorithm}
For the analysis, we index the iterations by $r = 0,1,\ldots,R$, where $R = \frac12 \log_2 \left( \frac{T}{\H kA  \log T }\right)$. So,   the accuracy parameter in iteration $r$ is $\epsilon = 2^{-r}$.  
Let $\mathcal{A}^{(r)}$ denote the action sets at the beginning of iteration~$r$.  
The total number of episodes over all iterations is
\[
\sum_{r=0}^{R} 12\, \H kA \log T \cdot 2^{2r}
\leq O(T).
\]

By a union bound argument, with probability at least $1 - \tfrac{1}{T}$,  
Lemma~\ref{lem:shrink key lemma} holds for all $R$ phases.  
We condition on this event for the remainder of the analysis.  
By part~1 of Lemma~\ref{lem:shrink key lemma}, the optimal policy is contained in  
$\mathcal{A}^{(r)}$ for every $r \ge 0$.  
Thus, we may apply Lemma~\ref{lem:shrink key lemma} in every iteration.

By part~2 of Lemma~\ref{lem:shrink key lemma},  
the regret incurred in any single episode of iteration $r$ is at most
$3\H C_{1,1} \cdot 2^{-r+1}.$ Moreover, the number of episodes in iteration $r$ is at most
$12\, \H kA \log T \cdot 2^{2r}.$

Therefore, with probability $1 -\frac{1}{T}$ the expected regret is upper bounded by:
\begin{align*}
\sum_{r=0}^{R} 
&\Bigl(12\, \H kA \log T \cdot 2^{2r}\Bigr)
\Bigl(3\H C_{1,1} \cdot 2^{-r+1}\Bigr) + T\cdot 3\H C_{1,1}2^{-R+1}\\
&= 
O\left(\H C_{1,1}\right)\left(\!\left(\H k A \log T\right)
\sum_{r=0}^{R} 2^r +T\cdot 2^{-R}\right)
\\
&= 
O\left(\H C_{1,1}\right)\left(\!\left(\H k A \log T\right)
2^R +T\cdot 2^{-R}\right)\\
&= O\left( \H C_{1,1}\sqrt{\H kA T \log T } \right) = O\left(\H^2(kA)^{\H}\sqrt{\H kA T \log T }\right)
\end{align*}

By a standard high probability regret to expected regret argument, the overall expected regret is $O\left(\H^2(Ak)^{\H}\sqrt{\H kA T \log T }\right) .$

\subsection{Ordered MDPs with Bandit Feedback}

Motivated by applications in stochastic optimization,
we now consider a class of structured MDPs which we refer to as \emph{ordered MDPs}. In this setting, we assume $k \leq \H$ and that 
the dynamics are restricted to allow only \emph{downward} transitions.  This setting models selection problems where the level represents the available capacity of some resource and actions correspond to utilizing  this resource (which means that the capacity can only decrease). The initial state of an ordered  MDP is  $(k, 1)$, i.e., in level $k$ and stage $1$. 
\begin{assumption}\label{def:ordered MDP}
For any action $a$, stage $i\in [\H]$ and levels $s > l$, we have  $p_i(s \mid l, a) = 0$.
\end{assumption}

  Furthermore, the agent is assumed to know  a total-ordering  on  actions that corresponds to  the probability of staying at the current level. 
\begin{assumption}\label{def:maximal action}
For any state $(l,i)$, there is a known total ordering $\prec_{l,i}$ on the actions $A$ such that:
$$
 p_{i}(l \mid l , a) \le  p_{i}(l \mid l , b) \qquad \mbox{ for all actions } a \prec_{l,i} b.
$$
\end{assumption}

\begin{theorem}\label{thm:ordered-mdp}
    There is an online learning algorithm for ordered MDPs with  unknown transition probabilities and  bandit feedback having regret $O( ( \frac{2e\H^2A}{k} )^{k}  \sqrt{\frac{k^3\H }{A}T \log T } )$.    

\end{theorem}
When $k\ll \H$, this regret bound is much better than the one for general MDPs (Theorem~\ref{thm:main-mdp}).

To prove \Cref{thm:ordered-mdp}, we modify Algorithm~\ref{alg:unifexp} to take advantage of this additional structure. We only modify the \textbf{Sampling Action} part in the algorithm and keep the \textbf{Learning} part essentially the same.

In the sampling step for general MDPs (Algorithm~\ref{alg:unifexp}), the algorithm uniformly samples actions at every state. 
For ordered MDPs, we adopt a different sampling strategy. 
Specifically, at each  state $(l,i)$, with probability $\frac{k}{2 \H}$ we perform uniform sampling over the available actions, 
and with the remaining probability $1 - \frac{k}{2\H}$ we deterministically choose the ``maximal'' action  (w.r.t.\ the total ordering $\prec_{l,i}$). 
This strategy increases the likelihood of remaining at the same level and  provides better control over the probability of reaching a particular state.

\begin{algorithm}
\caption{Ordered Explore then Refine Algorithm $\mathsf{OrderedExpRef} $}
\label{alg:sl}
\begin{algorithmic}[H]
  \State \textbf{Input:} action sets  $\{ A_{l,i} : l\in[k], i\in [\H]\}$   and  $ \epsilon >0 $ 
\For{$(l,i) \in [k] \times [\H]$} \Comment{\textbf{Sampling Actions}}
  \begin{itemize}
      \item With probability $\frac{k}{2\H}$:  sample $e_{l,i} $ uniformly from $A_{l,i}$ . 
      \item With probability $1 -\frac{k}{2\H}$: set $e_{l,i} = \arg\max_{a \in A_{l,i}}  p_i(l| l,a)$
  \end{itemize}
  \EndFor
  \For{$i = \H,\H-1 \dots 1$}\Comment{\textbf{Learning}}
    \For{$l=k \dots 1$}
      \For{ $a\in A_{l,i}$ }
    \State Let $\ba =  (\e_{k,i}, \dots  \e_{l+1,i},a, \e_{l-1,i}\dots \e_{1,i} )$.
    \State Play policy \([\mathbf{e_1}, \dots, \mathbf{e_{i-1}}, \mathbf{a},\bhalpha_{i+1}, \dots, \bhalpha_{\H}] \) for  \(\frac{12\log T}{\epsilon^2}\) episodes
    \State Let  \(\hat{\Phi}_{l,i}(a)\) be the empirical average reward for this policy.
    \EndFor
      \State Let  $\halpha_{l,i} = \arg\max\limits_{a\in A_{l,i}} \, \hat{\Phi}_{l,i}(a)$ be the empirical best action at state $(l,i)$.

       \State Let the refined action set be
$$ N_{l,i} =  \left\{a \in A_{l,i}  \Bigm|  
               \hat{\Phi}_{l,i}(a)  \ge  \hat{\Phi}_{l,i} (\halpha_{l,i}  ) -  C_{l,i}\,\cdot \epsilon
              \right\},\,\, \mbox{ where   $C_{l,i}$ is defined in \eqref{eq:ordered C_l,i}.}$$ 
              
     \EndFor
    \State     Let  $\bhalpha_{i} = (\halpha_{k,i} \dots \halpha_{1,i})$ be the vector of empirical best actions in stage $i$.
  \EndFor
  
  \State \textbf{Output}  refined action set $ N_{l,i}$ for all $(l,i) \in [k] \times [\H]$.
\end{algorithmic}
\end{algorithm}

Again, we use $\eta=( e_{l,i})_{l\in [k],i\in [\H]}$ to  denote the randomized  policy that uses the sampled actions. For any state $(l,i)$, we use  $Q_{l,i}$ to denote the visitation probability of $(l,i)$ under policy $\eta$.

\begin{lemma}\label{lem:order sampling}
For any $i\in [\H]$ and $l\in [k]$, we have:
 \begin{equation*}
    Q_{s,i+1} \,\, \geq  \,\, \max_{a\in A_{l,i}} \frac{k }{2\H A}\cdot Q_{l,i}\cdot   p_i(s|l,a) , \qquad \forall s<k \,\, \mbox{and}, 
\end{equation*}   

 \begin{equation*}
    Q_{l,i+1}  \,\, \geq \,\,  \max_{a\in A_{l,i}} \left(1- \frac{k}{2\H }\right)\cdot Q_{l,i}\cdot   p_i(l|l,a)  \,\, \geq \,\,  \max_{a\in A_{l,i}}e^{-\frac{k}{\H }} \cdot Q_{l,i}\cdot  p_i(l|l,a)  .
\end{equation*} 
\end{lemma}

\begin{proof}

We can prove the first inequality using the same argument as in Lemma~\ref{lem: unif sampling}. 
The only point of difference  is that, in Algorithm~\ref{alg:unifexp}, uniform sampling is applied with probability $1$. 
In Algorithm~\ref{alg:sl}, uniform sampling is used only with probability $\frac{k}{2\H}$, which accounts for the   extra $\frac{k}{2\H}$ factor.

For the second inequality ($s=l$), we consider a different way to reach state $(s,i+1)$:  
we choose the maximal action $e^* = \arg\max_{a \in A_{l,i}}  p_i(l| l,a)$. Then
\allowdisplaybreaks
\begin{align*}
    Q_{l,i+1} 
    &= \Pr\bigl[ \eta \text{ reaches } (l,i+1) \bigr] \\
    &\geq 
    \Pr\!\left[
        \eta \text{ reaches } (l,i)
        \,\bigwedge\, e_{l,i} = e^*
        \,\bigwedge\, e^* \text{ transitions to } (l,i+1)
    \right] \\
    &= 
    \Pr\bigl[ \eta \text{ reaches } (l,i) \bigr]
    \cdot
    \Pr[e_{l,i} = e^*]
    \cdot
    \Pr\!\left[e^* \text{ transitions to } (l,i+1)\right] \\
    &\geq 
    Q_{l,i} \cdot \left(1 - \tfrac{k}{2\H }\right)\cdot 
    \Pr\!\left[e^* \text{ transitions to } (l,i+1)\right]  =     \left(1 - \tfrac{k}{2\H }\right)\cdot Q_{l,i} \cdot \max_{a\in A_{l,i}} p_i(l|l,a)  \\
    & \ge e^{-k/\H }\cdot  Q_{l,i} \cdot \max_{a\in A_{l,i}} p_i(l|l,a) .
\end{align*}
The last inequality uses $k\le \H $ and $1-x/2\ge e^{-x}$ for $x\in [0,1]$. 
\end{proof}

As with the algorithm for general MDPs, we  introduce a sequence of constants that determine how we refine the action sets. However, the constants for  ordered MDPs are based on a different rule.
For any $(l,i) \in [k] \times [\H]$, define
 \begin{equation} \label{eq:ordered C_l,i}
C_{l,i}= e^{(\H-i)\frac{k}{\H } } \cdot \left(\frac{2A\H }{k}\right)^{l-1} \cdot (\H-i+1)^l = \e^{\frac{\H -i}{t}}\cdot  (2At)^{l-1}\cdot (\H-i+1)^l    
\end{equation}
where we use $t:= \frac{2\H }{k}\ge 1$. Note that $C_{\star,\H }=1$ and  the start state has $C_{k,1}\le (\frac{2eA\H ^2}{k})^k$.    These constants  satisfy the the following   lemma (which is identical to Lemma~\ref{lem:unif-induction} for the general case). 
\begin{lemma}\label{lem:order-induction}
    For any $(l,i) \in [k] \times [\H]$, and $a \in A_{l,i}$ we have:
    $$ 
    \frac{1}{Q_{l,i}}       +   \sum_{s \in [k]} p_i(s \mid l, a)   \frac{C_{s,i+1}   }{Q_{s,i+1} }   \leq \frac{C_{l,i} }{Q_{l,i}}  .
    $$
\end{lemma}
\begin{proof}

To simplify notation, we define  \( j := \H - i  \).
\allowdisplaybreaks
\begin{align*}
&\frac{1}{Q_{l,i}}       +   \sum_{s \in [k] } p_i(s \mid l, a)   \frac{C_{s,i+1}  }{Q_{s,i+1} }    \\
&=\frac{1}{Q_{l,i}}       +  \sum_{s \leq l} p_i(s \mid l, a)   \frac{ e^{\frac{j-1}{t}}(2At)^{s-1} j^s   }{Q_{s,i+1}}   \qquad  \text{(using the ordered property and \eqref{eq:ordered C_l,i})} \\
&= \frac{1}{Q_{l,i}}  + p_i(l \mid l, a)  \cdot  \frac{ e^{\frac{j-1}{t}}(2At)^{l-1}j^l   }{Q_{l,i+1} }    +  \sum_{s =1}^{l-1} p_i(s \mid l, a)  \cdot  \frac{e^{\frac{j-1}{t}}(2At)^{s-1}j^s   }{Q_{s,i+1} }    \\
&\leq \frac{1}{Q_{l,i}}+ \frac{e^{\frac{j-1}{t}} }{Q_{l,i}} \left(  e^{\frac{1}{t}}  \left( 2At\right)^{l-1} j^l + \sum_{s =1}^{l-1}2At   \left(2At\right)^{s-1} j^s \right)    \qquad \text{(both parts of Lemma~\ref{lem:order sampling})} \\
&\leq \frac{1}{Q_{l,i}}+ e^{\frac{j}{t}} \frac{(2At)^{l-1} }{ Q_{l,i}} \left(   j^l  + \sum_{s =1}^{l-1}  j^s  \right)   \qquad \text{(using $At\ge 1$)}\\
&\leq  e^{\frac{j}{t}}  \frac{(2At)^{l-1} }{ Q_{l,i}} \left(     j^l  + \sum_{s =1}^{l-1}  j^s  +1 \right)    \,\,\leq\,\, \frac{e^{\frac{j}{t}} (2At)^{l-1} (j+1)^l }{Q_{l,i}}    \,\,=\,\, \frac{C_{l,i}  }{Q_{l,i}}   
\end{align*}
\end{proof}

The rest of the analysis is identical to that for general MDPs in \S\ref{subsec:General MDP} (we just use Lemma~\ref{lem:order-induction} instead of \ref{lem:unif-induction}). 

Using the fact that the initial state is $(k,1)$,  we now obtain expected regret of $O\left(\H C_{k,1} \sqrt{\H  kA T \log T } \right) = O\left( \left( \frac{2\e A \H ^2 }{k} \right)^{k}  \sqrt{\frac{k^3 \H }{A}T \log T } \right) $ as the final regret bound. This proves the upper bound in Theorem~\ref{thm:ordered-mdp}.

\subsection{Applications}

In this section, we discuss   applications of ordered MDPs. 

{\bf k-item Prophet Inequality.}
The Prophet Inequality~\citep{krengel1977SemiamartsAF,samuel1984comparison} is a foundational problem in optimal stopping theory and has found extensive applications in algorithmic game theory. 
In this setting, an instance consists of $\H$ independent discrete random variables that arrive in a fixed sequence $X_1, \ldots, X_{\H}$, each with support size $A$. We consider policies that, upon observing the realization of each $X_i$, must immediately decide whether to select or discard it. The policy may select at most $k$ variables and terminates immediately upon making its $k$-th selection. The goal is to maximize the total reward of the selected variables.

We model this problem as an MDP with horizon $\H$, width $k$  and $A$ actions. The value function of the MDP is as follows. For each $(l,i) \in [k] \times [\H]$, let $V_{l,i}^*$ denote the optimal total expected rewards from $X_i, \dots, X_{\H}$ given that we have $l$ capacity left. An action $a$ in $(l,i)$ is a threshold for $X_i$: if the realization exceeds the threshold, then select $X_i$; otherwise, it is discarded. The action space corresponds to the support of $X_i$, which has size $A$. This MDP satisfies the ordered MDP assumption. At $(l,i)$, it can only transit to $(l-1,i+1)$ (if we select $X_i$) or $(l,i+1)$ (if we discard $X_i$). Hence, it satisfies Assumption~\ref{def:ordered MDP}. Moreover, for any action $a$, the probability of transiting to $(l,i+1)$ is $\Pr(X_i \leq a)$. Hence, for any two actions $a \geq b$, we have $\Pr(X_i \leq a) \geq \Pr(X_i \leq b)$. Hence,  Assumption~\ref{def:maximal action} is  satisfied with the total ordering on actions being the usual ordering of the support $A$.

{\bf Sequential Posted Pricing.}
A seller offers a  service to $\H$ customers arriving in a fixed order. Each customer $i$'s valuation is modeled as a discrete independent random variable $X_i$ . 
We use $A$ to denote  the maximum support size.  The seller can provide  service to at most $k$ customers.
We consider Sequential Posted Pricing (SPP) mechanisms,  introduced by \citet{ChawlaHMS10}. When any  customer $i$ arrives, they are offered a  take-it-or-leave-it price  $p_i$. The policy needs to choose the prices $p_i$ for each customer. The seller collects revenue from all customers who accept their offered price. We can  model SPP   as an MDP with horizon $\H$, width $k$  and $A$ actions. At any state $(l,i)$, the value function corresponds to the maximum expected revenue from customers $i,i+1,\cdots \H$ using capacity $l$. Each action corresponds to a price to offer customer $i$, which  equals one of the $A$ values in the support  of $X_i$ (without loss of generality). The transitions out of $(l,i)$ are as in the prophet inequality case: to $(l-1,i+1)$ if customer $i$ accepts, and to $(l,i+1)$ if  customer $i$ rejects. So, Assumption~\ref{def:ordered MDP} holds. Assumption~\ref{def:maximal action} also holds for the same reason as for prophet inequality.   

{\bf Stochastic Knapsack.}
This is a classic problem in stochastic optimization, first introduced in \citep{DGV08} and extensively studied since then \citep{BGK11,GuptaKMR11,Ma18}. 
There are $\H$ items with integer random rewards $\{r_i\}_{i=1}^{\H}$ and costs $\{C_i\}_{i=1}^{\H}$, which are considered in a {\em fixed} order in our setting. The reward and cost of any item   can be correlated. But these values are independent across different items. When an item $i$ is considered, the policy needs to either select or reject it. If  the item is rejected then the policy skips to the next item $i+1$. If the item is selected then the policy observes the $r_i,C_i$ realization and then continues to item $i+1$. 
Given a knapsack budget $k$, a policy sequentially selects/rejects items until the total cost exceeds $k$. The objective is to maximize the expected total reward from selected items that fit into the knapsack. If an item  overflows the budget, it does not contribute to the objective. We can model this as an MDP with horizon $\H$, width $k$  and $A=2$ actions (accept or reject). Each state $(l,i)$ corresponds to starting the policy at item $i$ with remaining budget $l$. Note that the transitions out of $(l,i)$ are only to $(s,i+1)$ where $s\le l$ as the costs are non-negative. Moreover, the reject action at state $(l,i)$ transitions to $(l,i+1)$ w.p. $1$: so ``reject''  is the maximal action at every state. Hence, Assumptions~\ref{def:ordered MDP} and \ref{def:maximal action} are satisfied.

\section{Lower Bounds}\label{sec: lower bnds}

In this section, we will establish regret lower bounds for bandit learning of  MDPs.
First, we show the lower bound for general MDP with $k$ levels and $\H$ stages.
 
\begin{restatable}{theorem}{thmMDPlowerbnd}\label{thm: MDP lowerbnd}
Any  bandit-feedback   learning algorithm for  MDPs with $k=2$ levels and horizon-length $\H$    has  regret   $\Omega\left(\min\left\{\sqrt{ A^\H T} \,,\, T\right\}\right)$. 
\end{restatable}

\def\cA{\ensuremath{{\cal A}}\xspace}
 \def\bt{\bm{\theta}}
\begin{proof}
We construct a family of hard MDP instances $\{M^{\bt} | \bt \in A^n\}$ that allows us to  reduce from the usual  
MAB problem with $A^n$ arms and Bernoulli rewards.

All our MDPs involve $k=2$ levels and $\H$ stages. The start state is $(1,1)$. There is reward only at the final stage $\H$ (at all other states the rewards are zero).  
For a fixed vector $\bt= \langle \theta_1,\cdots \theta_n\rangle \in A^n$, the MDP $M^{\bt}$ is defined as follows:
\begin{itemize}

  \item For any $i\in [n-1]$, the transitions at state $(1,i)$ are as follows. Under action $\theta_i$, we transition to state $(1,i+1)$ with probability one. Under any other action $a\in A\setminus \theta_i$, we  transition to state $(2,i+1)$ with probability one.
  \item For any $i\in [n-1]$, the transitions at state $(2,i)$ are as follows. Under any action in $ A$, we  transition to state $(2,i+1)$ with probability one.
  \item At the final stage $n$, if action $\theta_n$ is chosen at state $(1,n)$ then  we obtain  reward of  $\text{Bern}(\frac{1}{2} + \epsilon)$. In all other cases (any action in $A\setminus \theta_n$ at  state $(1,n)$ or any action at state $(2,n)$), we obtain reward of  $\text{Bern}(\frac{1}{2})$. The value of $\epsilon > 0$ will be specified later.
\end{itemize}
Observe that if we ever get to a state in level $2$ then we will remain in that level until the end and only receive   $\text{Bern}(\frac{1}{2})$ reward.

A policy for any such MDP $M^{\bt}$ can be  specified by a vector $\bm{\pi} \in A^n$, where for any stage $i\in [n]$, the action at state $(1,i)$ is $\pi_i$ and the action at any state $(2,i)$ is arbitrary. Moreover, the observed reward under policy $\bm{\pi}$ is Bernoulli with mean:
\[
\mu(\bm{\pi}) =
\begin{cases}
\frac{1}{2} + \epsilon & \text{if } \pi_i = \theta_i, \text{for all $i \in [n]$} \\
\frac{1}{2} & \text{otherwise}.
\end{cases}
\]

Since we are in the bandit setting, the  only observation  upon choosing policy $\bm{\pi}$ is the final reward. Therefore,  bandit-feedback learning for this family of MDPs is equivalent to stochastic Bernoulli MAB with $N=A^n$ arms (each arm corresponds to a choice of $\bt$)  where the optimal arm has mean $\frac12+\epsilon$ and all other arms have mean $\frac12$. 

 We can directly use lower bound results for this setting (Theorem 5.1 in \cite{AuerPeterCesa2002nonstochastic}), which implies the following. For any online algorithm, there is some MDP instance  $M^{\bt^*}$ on which the regret is at least $\Omega(\sqrt{NT})$. Using $N=A^n$, this completes the proof.
 \end{proof}

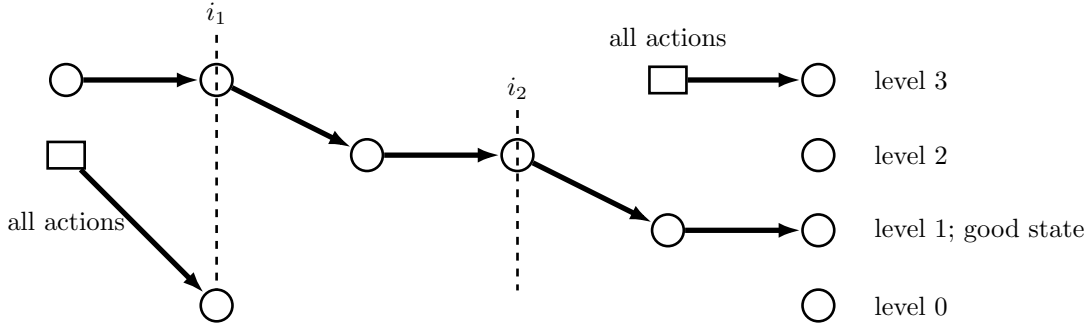
\begin{figure}[t] 
    \centering
\begin{tikzpicture}[scale=1, transform shape,
    line width=1pt, 
    state/.style={circle, draw, minimum size=12pt, inner sep=1pt},
    font=\small,
    >={Latex[length=3mm, width=2mm]} 
]

\node[state] (a1) at (0,4) {};
\node[state] (a2) at (2,4) {};
\node[state] (a12) at (2,1) {};
\node[state] (a3) at (4,3) {};
\node[state] (a4) at (6,3) {};
\node[state] (a5) at (8,2) {};
\node[state] (a6) at (10,2) {};
\node[state] (a9) at (10,1) {};
\node[state] (a10) at (10,4) {};
\node[state] (a11) at (10,3) {};
\node[rectangle, draw, minimum width=14pt, minimum height=10pt] (a7) at (0,3) {};
\node[rectangle, draw, minimum width=14pt, minimum height=10pt] (a8) at (8,4) {};

\draw[->, line width=2pt] (a1) -- (a2);
\draw[->, line width=2pt] (a2) -- (a3);
\draw[->, line width=2pt] (a3) -- (a4);
\draw[->, line width=2pt] (a4) -- (a5);
\draw[->, line width=2pt] (a5) -- (a6);
\draw[->, line width=2pt] (a7) -- (a12);
\draw[->, line width=2pt] (a8) -- (a10);

\draw[dashed] (2,4.6) -- (2,1.2);
\draw[dashed] (6,3.6) -- (6,1.2);

\node[right=6mm] at (a6) {level $1$; good state};
\node[right=6mm] at (a11) {level $2$};
\node[right=6mm] at (a9) {level $0$};
\node[right=6mm] at (a10) {level $3$};
\node[above=6mm] at (a2) {\textbf{$i_1$}};
\node[above=6mm] at (a4) {\textbf{$i_2$}};
\node[below=6mm] at (a7) {all actions};
\node[above=3mm] at (a8) {all actions};
\end{tikzpicture}

    \caption{Ordered MDP instance with $k=2$ and $\H =6$.  
    Path $\bm{\tau}$ is from $(1,4)$ to $(6,1)$; it  transitions downward  at stages $2$ and $4$. 
    }
    \label{fig:Hard Instance}
\end{figure}

\begin{restatable}{theorem}{thmSTOlowerbnd}\label{thm: sto order lower bound}
  Any bandit-feedback learning algorithm for  stochastic ordered MDPs (with $k+2$ levels, $\H$ stages and $A$ actions)    has  regret at least  $ \Omega\left(\min\left\{\sqrt{  \binom{\H}{k}A^k T}  \,,\, T\right\}\right)$.
\end{restatable}
 \begin{proof} 
Consider a  class of ordered MDPs with $\H $ stages, $k+2$ levels and $A+1$ actions.  All transitions will be non-increasing in levels (as required in the setting of ordered MDPs). We index the actions by  $0, 1,2,\cdots A$, which will also be  the total-ordering for  Assumption~\ref{def:maximal action}  at  all states. We will use $\cA=\{1,2,\cdots A\}$ to denote all actions except the maximal one.  
We index the levels as $\{0,1,2,\cdots k+1\}$ and the initial state is $(k+1, 1)$ in stage $1$. 
Rewards are generated only at states in the final stage $\H $ (referred to as terminal states). Among the terminal states, state $(1,\H)$ is the \emph{good state}, where the reward is drawn from $\Bern\!\left(\frac{1}{2}+\epsilon\right)$. All other terminal states are \emph{bad} states, where the reward is  $\Bern\!\left(\frac{1}{2}\right)$. The parameter $\epsilon$ will be specified later. All transition probabilities are $0-1$ valued.

\def\td{\ensuremath{\bm{\tau}_{d}}\xspace}
An ordered MDP instance is given by a tuple $(\bm{\tau}, \bm{a})$ where: 
\begin{itemize}
    \item $\bm{\tau}$ is a path from the initial state $(k+1,1)$ to the good state $(1,\H)$. Formally, $\bm{\tau}$ is specified by a sequence of $k$ stages $0<i_1 < \dots < i_k<\H$ where the path reduces its level. To keep notation simple, we also use $i_0=0$ and $i_{k+1}=\H$.
    The path $\bm{\tau}$  consists of the following states:
    $$
  \bigcup_{p=0}^{k} \{(k+1-p, i_{p}+1) , \dots , (k+1-p,i_{p+1})\}.
$$ 
We refer to the states of path $\bm{\tau}$ at stages $i_1,\cdots i_k$ as {\em down} states and use \td to denote this set; the path reduces in level at these states. At all other states of $\bm{\tau}$, the path stays at the same level.   
Note that the number of possible paths $\bm{\tau}$ is $\binom{\H-1}{k}$. 

\item $\ba = (a_1 \dots a_{k}) \in \cA^k$ is a vector of actions. The number of possible vectors is $A^k$. 
\end{itemize}
\def\bmdp{\ensuremath{M_{\bm{\tau},\bm{a}}}\xspace}
\def\optp{\ensuremath{{\cal P}_{\bm{\tau},\bm{a}}}\xspace}

The MDP \bmdp has the following transitions  at any state $(l,i)$.

\begin{itemize}
\item 
If $(l,i)\in \td$ is a down state on  path $\bm{\tau}$, then 
$i=i_p$ for some $p\in [k]$. 
In this case, 
\begin{itemize}
    \item all actions $a<a_p$ transition ``horizontally'' to state $(l,i+1)$.
    \item action $a_p$ transitions ``diagonally'' to state $(l-1,i+1)$.
    \item  all actions $b>a_p$ transition ``vertically'' to state $(0,i+1)$.
    \end{itemize}
    Here, $a_p$ is the $p^{th}$ entry of the vector of actions $\ba$. 
\item 
If $(l,i)\in \bm{\tau}\setminus \td$ is any other state on path $\bm{\tau}$, then 
\begin{itemize}
    \item the maximal action $0$ transitions ``horizontally'' to state $(l,i+1)$.
     \item  all actions in \cA  transition ``vertically'' to state $(0,i+1)$.
    \end{itemize}  
\item If $(l,i)$ and lies \emph{above} the path, i.e., at stage $i$ its level $l$ is larger than the corresponding level of path $\bm{\tau}$. Then, all actions  transition ``horizontally'' to state  $(l,i+1)$.
    \item If $(l,i)$ and lies \emph{below} the path, i.e.,   at stage $i$ its level $l$ is smaller than the corresponding level on the $\bm{\tau}$. Then,  all actions transition ``vertically'' to state $(0,i+1)$. 

\end{itemize}
See Figure~\ref{fig:Hard Instance} for an example.  Note that Assumption~\ref{def:maximal action} is satisfied by the total order $0,1,\cdots A$ in all the cases.

In total we obtain $L=\binom{\H-1}{k}\cdot A^{k}$ MDP instances. 
Each of these instances has optimal value $\frac12 +\epsilon$. A policy for   instance \bmdp is optimal iff (1) it chooses action $a_p$ at each down state $(l,i_p)\in \td$ and (2) it chooses action $0$ at all other states of $\bm{\tau}$. We use $\optp $ to denote the set of optimal policies for instance \bmdp. 

\begin{observation}\label{obs:epsilon regret}
Any policy   not in $\optp$ has  expected reward of $\tfrac{1}{2}$ for MDP instance \bmdp.
\end{observation}
To see this, let $\pi\not\in \optp$ be any such policy. It suffices to show   that policy $\pi$ will always terminate at a bad state. Consider the earliest stage $i$ that satisfies one of the following (there must be such a stage because $\pi\not\in \optp$).
\begin{itemize}
    \item Policy $\pi$ chooses an action $a\ne a_p$ at some down state $(l,i)\in \td$ where $i=i_p$. In this case,  with probability $1$, $\pi$ traverses path $\bm{\tau}$ until state $(l,i)$. As $a\ne a_p$, policy $\pi$  transitions from state $(l,i)$ to either $(l,i+1)$ or $(0,i+1)$, after which it must end up at a bad state.  
    \item Policy $\pi$ chooses an action $a\ne 0$ at some  state $(l,i)\in \bm{\tau}\setminus  \td$. Again,  with probability $1$, $\pi$ traverses path $\bm{\tau}$ until state $(l,i)$. As $a\ne 0$, policy $\pi$ then transitions to state $(0,i+1)$, after which it must end up at the bad state $(0,\H)$. 
\end{itemize}

\begin{observation}\label{obs:disjoint opt}
    For any two distinct  instances $(\bm{\tau},\ba)\ne  (\bm{\sigma},\bm{b})$, we have $\optp \bigcap {\cal P}_{\bm{\sigma},\bm{b}} = \emptyset$.
\end{observation}
Suppose first that $\bm{\tau}\neq\bm{\sigma}$. Then, consider the first state $(l,i)$ after which the two paths diverge: say $\bm{\tau}$ stays at the same level and $\bm{\sigma}$ reduces its level. Then, every policy in $\optp$ chooses action $0$ at state $(l,i)$ whereas every   policy in ${\cal P}_{\bm{\sigma},\bm{b}}$ chooses action $b_p\ne 0$ at  state $(l,i)$. So,  $\optp \bigcap {\cal P}_{\bm{\sigma},\bm{b}} = \emptyset$.
Now, suppose $\bm{\tau}=\bm{\sigma}$. Then, we must have $\bm{a}\ne \bm{b}$; let $p\in[k]$ be any index with $a_p\ne b_p$. Let $(l,i)\in \bm{\tau}_d$ be the $p^{th}$ down state in path $\bm{\tau}$. Then, every policy in $\optp$ chooses action $a_p$ at state $(l,i)$ whereas every   policy in ${\cal P}_{\bm{\sigma},\bm{b}}$ chooses action $b_p\ne a_p$ at this state. So,  $\optp \bigcap {\cal P}_{\bm{\sigma},\bm{b}} = \emptyset$ again.

In the following we will use both $r \in[L]$ and $(\bm{\tau},\ba)$ to index an instance. We also define the $0^{th}$  MDP instance to be the  ``uniform'' one  where {\em all}  terminal states (including the good state) generate $\Bern\!\left(\frac{1}{2}\right)$ reward. Note that any policy for the $0^{th}$ MDP instance has expected reward $\frac12$. 

Now, let's fix an arbitrary learning algorithm $\Pi$. We will show a regret  lower bound for $\Pi$. At a high level, the analysis is similar to Theorem 5.1 in \cite{AuerPeterCesa2002nonstochastic}. We use $\Pr_r(\cdot)$ and $\E_r[\cdot]$ to denote the probability and expectation induced by $T$ episodes of interaction of $\Pi$ with the $r^{th}$  instance, where   $0 \leq r \leq L$. Let  $m_r(T) = \sum_{t=1}^T \mathbf{1}\{ \Pi^t\in {\cal P}_r\}$, where $\Pi_t$ denotes the policy chosen by $\Pi$ in the  $t^{th}$  episode. Note that $m_r(T)$ is the number of episodes during which  algorithm $\Pi$ chooses a  policy that is  optimal  for the $r^{th}$ MDP instance. 

We first assume the following two inequalities hold and finish the proof using these. 
\begin{equation}\label{eq:KL of inst}
    \E_r[m_r(T)] \leq \E_0[m_r(T)] + T\sqrt{2\epsilon^2 \E_0[m_r(T)]}
\end{equation}
and 
\begin{equation}\label{eq:lowbnd disjoint}
    \sum_{r=1}^L \E_0[m_r(T)] \leq T
\end{equation}

Then we have:
\[
\sum_{r=1}^L \E_r[m_r(T)]  \leq  T + T\epsilon \sum_{r=1}^L \sqrt{2\E_0[m_r(T)]} \leq T +TL\epsilon \sqrt{2\sum_{r=1}^L\frac{ \E_0[m_r(T)]}{L}} 
     \leq  T + T\epsilon \sqrt{2TL}.
\]
where the second inequality uses Jensen's inequality for concave functions and the third inequality uses \eqref{eq:lowbnd disjoint}.
Let $R_r$ denote the expected regret of algorithm $\Pi$ under the $r^{th}$  instance. 
Using Observation~\ref{obs:epsilon regret} and  that the optimal value of any instance $r\ne 0$ is $\frac12+\epsilon$, we have $R_r=\epsilon \cdot  (T - \E_r[m_r(T)] )$. So, 
\[
\sum_{r=1}^L R_r 
= \sum_{r=1}^L \epsilon \cdot  (T - \E_r[m_r(T)] ) 
\geq \epsilon (TL - T - T\epsilon\sqrt{2TL} ).
\]

Now, choosing $\epsilon = \frac{1}{8} \cdot \min\left\{\sqrt{\frac{L}{T}}, \, 1 \right\}$, the last term above is $\Omega\left( \min\left\{ L \sqrt{LT},\, LT \right\}\right)$. Hence,  
\[
\max_{ r \in [L]} R_r  \geq \frac{\sum_{r=1}^L R_r }{L} \geq  \Omega\left(\min\left\{\sqrt{LT},\, T\right\} \right)=  \Omega\!\left(\min\left\{\sqrt{\binom{\H-1}{k}\cdot A^k T},\,T\right\}\right).
\]

\paragraph{Multi-Armed Bandit View} Now we prove \eqref{eq:KL of inst} and \eqref{eq:lowbnd disjoint}. We   view each policy $\pi\in (\cA\cup \{0\})^{k\times \H }$ as a distinct arm in a multi-armed bandit problem, so in total there are $(A+1)^{\H k}$ arms. We introduce some standard notation. Let $\Pr_r^\pi$ denote the reward distribution associated with policy (arm) $\pi$ in the $r^{th}$  instance. Define
$N_\pi(T) = \sum_{t=1}^T \mathbf{1}\{\Pi^t = \pi\},$
the number of times policy $\pi$ is pulled over $T$ episodes. Let $\KL(P, Q)$ denote the KL divergence between distributions $P$ and $Q$.  

By the chain rule of KL divergence (Lemma 15.1 in \cite{lattimore2020bandit}), we have:
\begin{align*}
    \KL(\pr_0,\pr_r) &= \sum_{\pi\,:\,policy}   \E_0 [N_\pi(T)] \KL(\pr_0^\pi,\pr_r^\pi)\\
    &= \sum_{\pi \not \in {\cal P}_r} \E_0[N_\pi(T) ]\KL\left(\Bern\left(\frac{1}{2}\right), \Bern\left(\frac{1}{2}\right)\right)+ \sum_{\pi \in {\cal P}_r} \E_0[N_\pi(T) ]\KL\left(\Bern\left(\frac{1}{2}\right),\Bern\left(\frac{1}{2}+\epsilon\right)\right) \\
    & = \E_0[m_r(T)] \cdot \KL\left(\Bern\left(\frac{1}{2}\right),\Bern\left(\frac{1}{2}+\epsilon\right)\right) = -\frac{1}{2} \ln(1-4\epsilon^2) \cdot \E_0[m_r(T)]  \\
    &\leq 4\epsilon^2 \cdot \E_0[m_r(T)]  
\end{align*}
The second equality uses Observation~\ref{obs:epsilon regret} and the fact that all policies in the uniform instance generate $\Bern(\frac{1}{2})$ reward.
The last inequality uses that $-\ln(1-x) \leq 2x$ when $x \in (0,\frac{1}{4})$, which is true by our choice of $\epsilon$. Now following the same analysis in Lemma A.1.~\cite{AuerPeterCesa2002nonstochastic}, we can prove \eqref{eq:KL of inst}. The analysis is the following: 
$$
\E_r[m_r(T)] - \E_0[m_r(T)] \leq T\cdot \delta(\pr_0,\pr_r)\leq T \sqrt{\frac{1}{2}\KL (\pr_0,\pr_r)}
$$
where $\delta(P,Q)$ is the total variation distance between two distributions $P$ and $Q$. The first inequality is from the total variation inequality And the second inequality is from Pinsker's inequality.

As for \eqref{eq:lowbnd disjoint}, by Observation~\ref{obs:disjoint opt}, we have:
$$
    \sum_{r=1}^L \E_0[m_r(T)] =  \sum_{r=1}^L \sum_{\pi\in {\cal P}_r} \E_0[N_\pi(T)] \leq  \sum_{\pi\,:\,policy} \E_0[N_\pi(T)] \leq T.
$$

\end{proof}

\section{Experiments}
We present empirical evaluations of MDPs under bandit feedback to corroborate our theoretical findings. The experimental settings are modeled as $k$-item prophet inequality. We test Algorithm~\ref{alg:unifexp}, Algorithm~\ref{alg:sl} and use UCB-VI\citep{AzarOM17} as the benchmark. We  highlight that our algorithms  receive very little feedback: only the cumulative  reward at the end of each episode. Whereas,   UCB-VI  receives feedback on all visited states and individual rewards  (semi-bandit feedback).

We consider two types of instances, denoted $I_1$ and $I_2$.
In $I_1$, all random variables share the same support $\left\{0, \frac{1}{A-1}, \frac{2}{A-1}, \dots, 1\right\}$, and each value in the support is assigned equal probability (i.e., the distribution is uniform).
In $I_2$, the random variables have distinct supports. For each random variable, we independently sample $A$ values uniformly at random from the interval $[0,1]$, and then assign probabilities to these values by sampling a probability vector uniformly at random from the $(A-1)$-dimensional simplex.

For both $I_1$ and $I_2$, we conducted experiments with the following parameter settings: 
(i) $\H = 15$, $A = 5$, and $k \in \{2, 3, 4\}$; and 
(ii) $\H = 20$, $A = 5$, and $k \in \{2, 3, 4\}$. 
In all cases, we ran $T = 2 \times 10^6$ episodes. 
To ensure comparability across different values of $k$, we normalize the cumulative regret by $k$, so that the maximum possible reward per episode is at most $1$.

The results are summarized in Figure~\ref{fig:prophet_regret_combined_all} and Table~\ref{tab:prophet_runtime}, \ref{tab:prophet_runtime_2}. The running time in $I_1$ and $I_2$ are almost the same. We only show the time in $I_1$. The experiments were conducted on MacBook Pro with Apple M3 chip and 16GB memory.  

\begin{table}[htbp]
\centering
\caption{Running times (in seconds) for $I_1$ type $\H =15, A= 4,k \in\{2,3,4\}$.}
\label{tab:prophet_runtime}
\begin{tabular}{lccc}
\toprule
\textbf{Algorithm} & $ k=2$ & $k=3$ & $k=4$ \\
\midrule
ExpRef    & 133 & 137 &139\\
OrderedExpRef &134 &138& 139 \\
UCB-VI     & 947 & 1336& 1723 \\
\bottomrule
\end{tabular}
\end{table}

\begin{table}[htbp]
\centering
\caption{Running times (in seconds) for $I_1$ type $\H =20, A= 4,k \in\{2,3,4\}$.}
\label{tab:prophet_runtime_2}
\begin{tabular}{lccc}
\toprule
\textbf{Algorithm} & $ k=2$ & $k=3$ & $k=4$ \\
\midrule
ExpRef    & 180 & 179 &183\\
OrderedExpRef &179 &180& 181 \\
UCB-VI     & 1161 & 1683& 2246 \\
\bottomrule
\end{tabular}
\end{table}

\begin{figure*}[ht]
    \centering

    \begin{subfigure}[b]{0.32\linewidth}
        \centering
        \includegraphics[width=\linewidth]{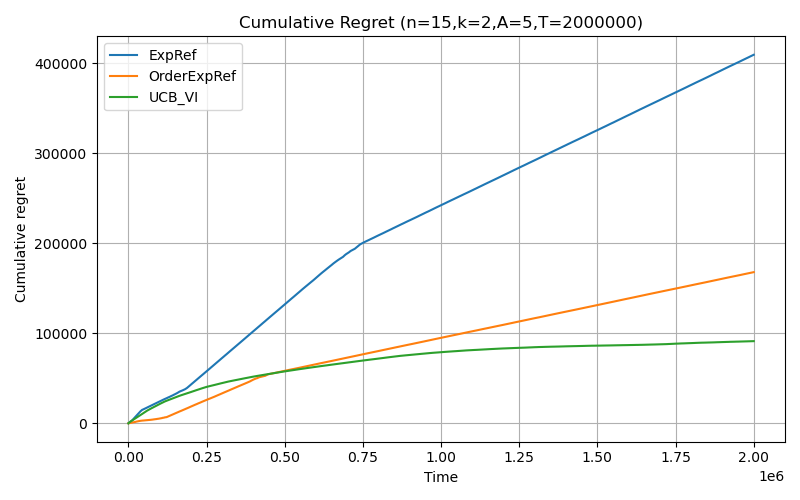}
        \caption{ $ \mathcal{I}_1 $ :  $ \H=15 $ ,  $ k=2 $ ,  $ A=5 $ }
        \label{fig:prophet_1a}
    \end{subfigure}
    \hfill
    \begin{subfigure}[b]{0.32\linewidth}
        \centering
        \includegraphics[width=\linewidth]{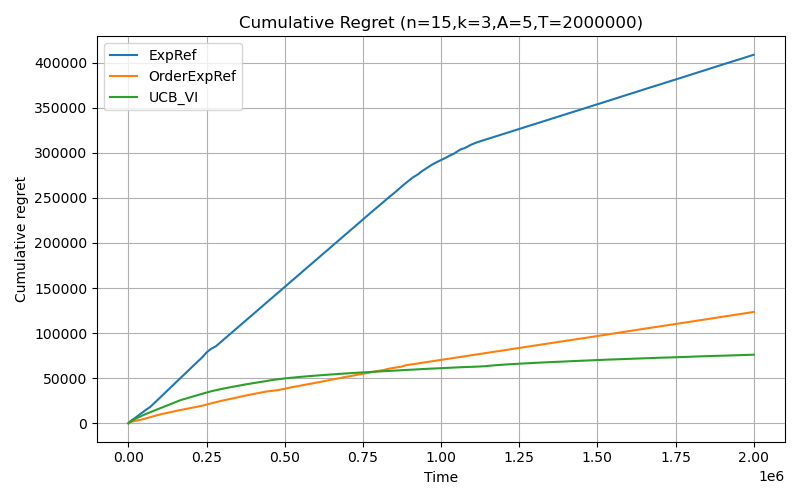}
        \caption{ $ \mathcal{I}_1 $ :  $ \H=15 $ ,  $ k=3 $ ,  $ A=5 $ }
        \label{fig:prophet_1b}
    \end{subfigure}
    \hfill
    \begin{subfigure}[b]{0.32\linewidth}
        \centering
        \includegraphics[width=\linewidth]{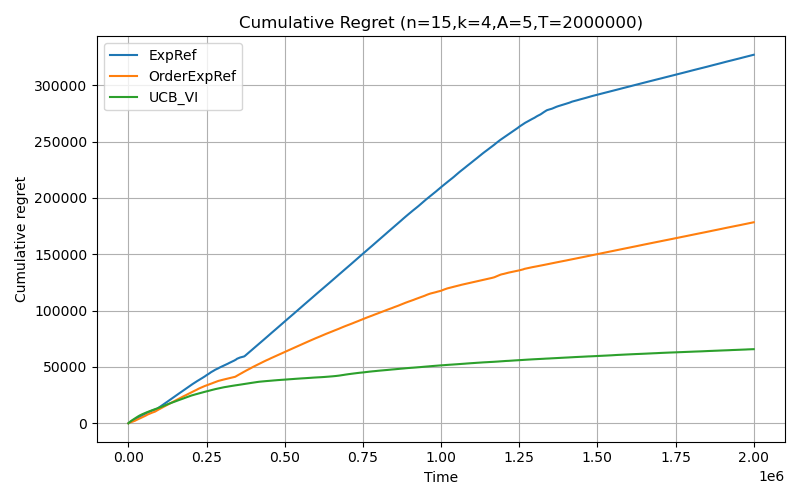}
        \caption{ $ \mathcal{I}_1 $ :  $ \H=15 $ ,  $ k=4 $ ,  $ A=5 $ }
        \label{fig:prophet_1c}
    \end{subfigure}

    \begin{subfigure}[b]{0.32\linewidth}
        \centering
        \includegraphics[width=\linewidth]{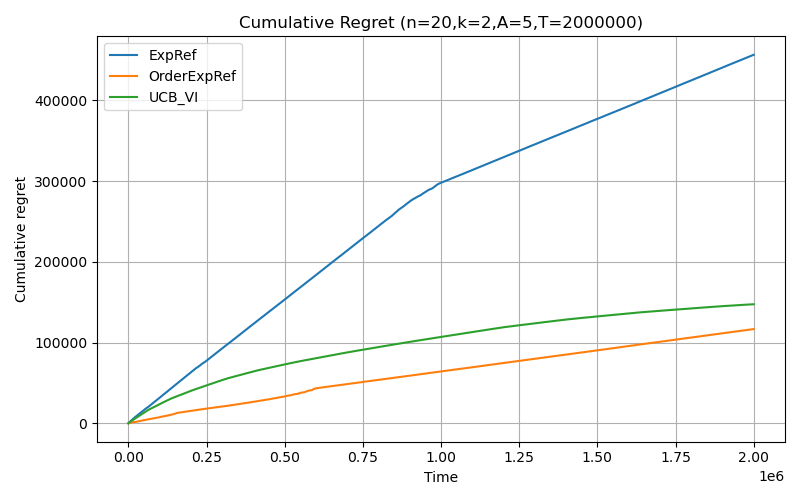}
        \caption{ $ \mathcal{I}_1 $ :  $ \H=20 $ ,  $ k=2 $ ,  $ A=5 $ }
        \label{fig:prophet_2a}
    \end{subfigure}
    \hfill
    \begin{subfigure}[b]{0.32\linewidth}
        \centering
        \includegraphics[width=\linewidth]{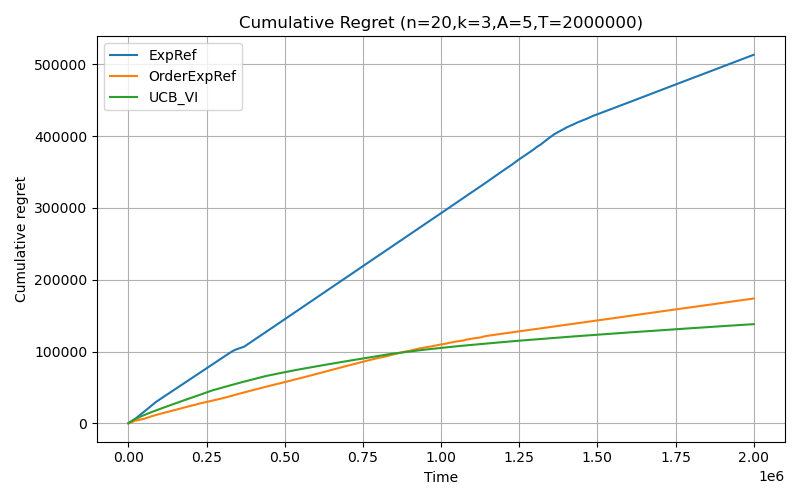}
        \caption{ $ \mathcal{I}_1 $ :  $ \H=20 $ ,  $ k=3 $ ,  $ A=5 $ }
        \label{fig:prophet_2b}
    \end{subfigure}
    \hfill
    \begin{subfigure}[b]{0.32\linewidth}
        \centering
        \includegraphics[width=\linewidth]{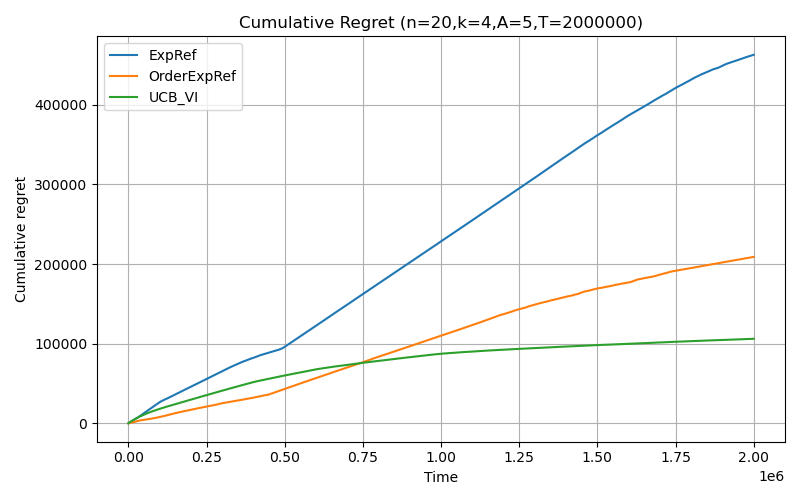}
        \caption{ $ \mathcal{I}_1 $ :  $ \H=20 $ ,  $ k=4 $ ,  $ A=5 $ }
        \label{fig:prophet_2c}
    \end{subfigure}

    \begin{subfigure}[b]{0.32\linewidth}
        \centering
        \includegraphics[width=\linewidth]{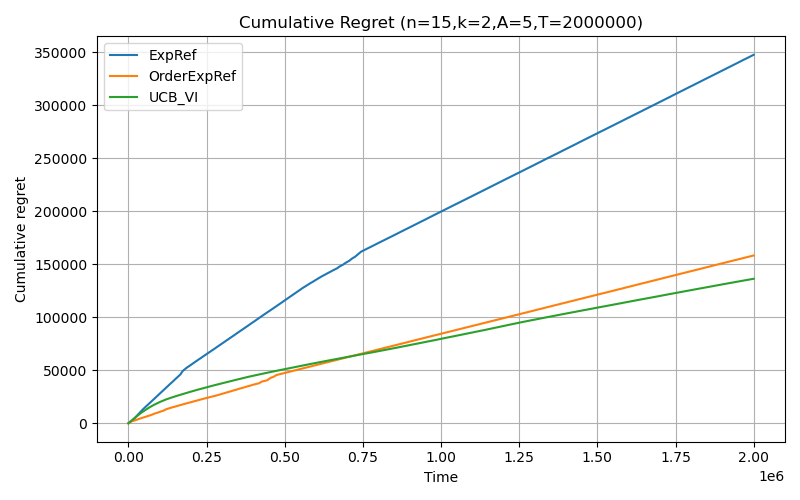}
        \caption{ $ \mathcal{I}_2 $ :  $ \H=15 $ ,  $ k=2 $ ,  $ A=5 $ }
        \label{fig:prophet_3a}
    \end{subfigure}
    \hfill
    \begin{subfigure}[b]{0.32\linewidth}
        \centering
        \includegraphics[width=\linewidth]{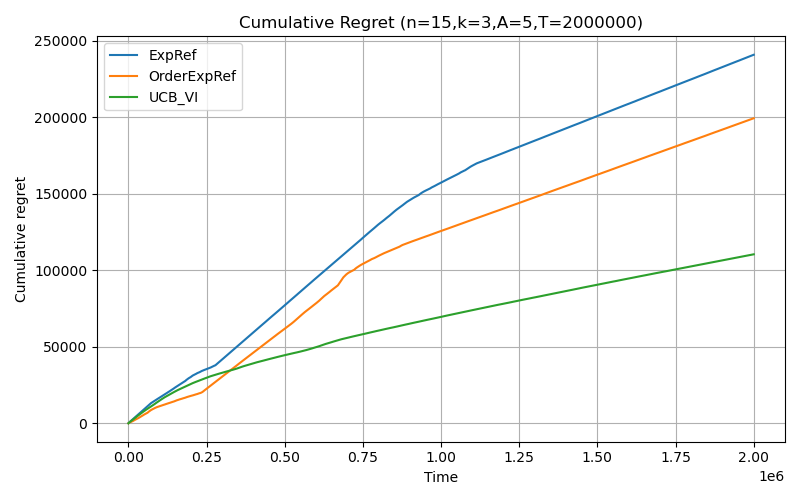}
        \caption{ $ \mathcal{I}_2 $ :  $ \H=15 $ ,  $ k=3 $ ,  $ A=5 $ }
        \label{fig:prophet_3b}
    \end{subfigure}
    \hfill
    \begin{subfigure}[b]{0.32\linewidth}
        \centering
        \includegraphics[width=\linewidth]{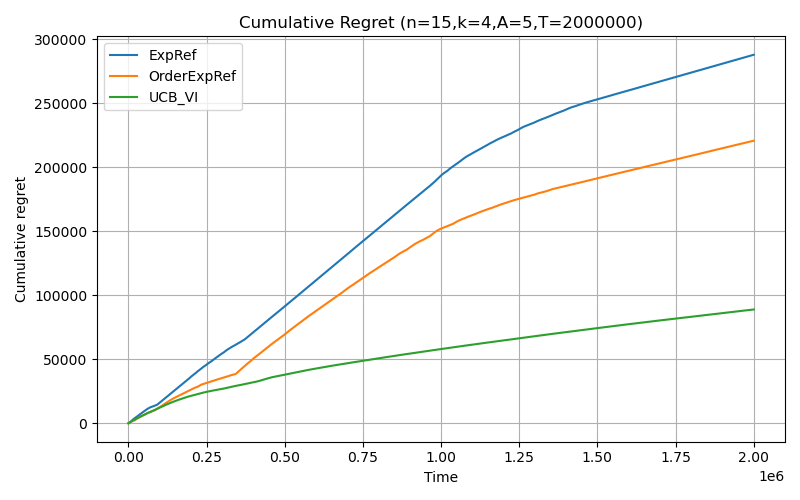}
        \caption{ $ \mathcal{I}_2 $ :  $ \H=15 $ ,  $ k=4 $ ,  $ A=5 $ }
        \label{fig:prophet_3c}
    \end{subfigure}

    \begin{subfigure}[b]{0.32\linewidth}
        \centering
        \includegraphics[width=\linewidth]{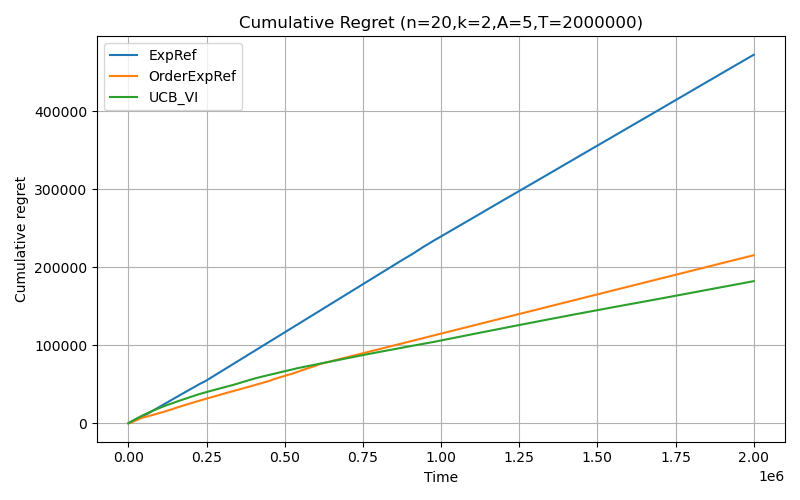}
        \caption{ $ \mathcal{I}_2 $ :  $ \H=20 $ ,  $ k=2 $ ,  $ A=5 $ }
        \label{fig:prophet_4a}
    \end{subfigure}
    \hfill
    \begin{subfigure}[b]{0.32\linewidth}
        \centering
        \includegraphics[width=\linewidth]{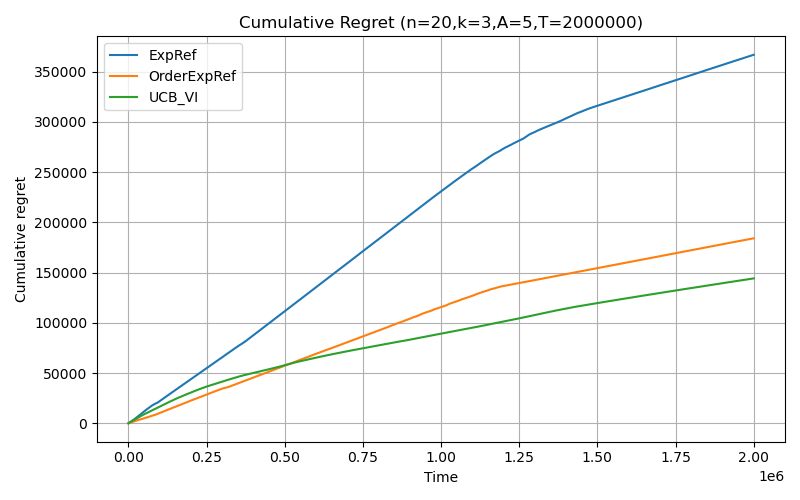}
        \caption{ $ \mathcal{I}_2 $ :  $ \H=20 $ ,  $ k=3 $ ,  $ A=5 $ }
        \label{fig:prophet_4b}
    \end{subfigure}
    \hfill
    \begin{subfigure}[b]{0.32\linewidth}
        \centering
        \includegraphics[width=\linewidth]{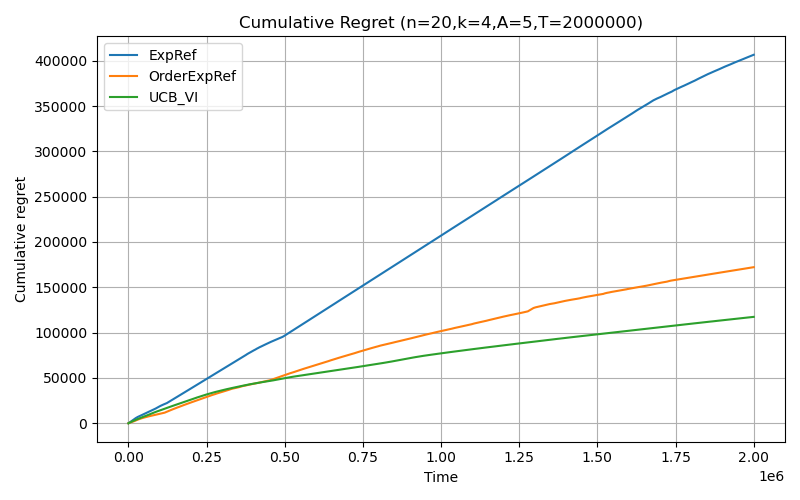}
        \caption{ $ \mathcal{I}_2 $ :  $ \H=20 $ ,  $ k=4 $ ,  $ A=5 $ }
        \label{fig:prophet_4c}
    \end{subfigure}

    \caption{Cumulative regret of algorithms on Prophet Inequality instances  $ \mathcal{I}_1 $  (top two rows) and  $ \mathcal{I}_2 $  (bottom two rows). Different values of  $ \H $  are shown in the first and second rows for each type respectively.}
    \label{fig:prophet_regret_combined_all}
\end{figure*}

\paragraph{Discussion} The empirical results show that \textsc{UCB-VI} achieves the lowest regret, reflecting its rapid convergence. Moreover, Algorithm~\ref{alg:sl} consistently outperforms Algorithm~\ref{alg:unifexp} in terms of regret. These findings are consistent with theoretical expectations: \textsc{UCB-VI} leverages substantially more information than the other two methods, while Algorithm~\ref{alg:sl} exploits the structural properties inherent to the ordered MDP formulation. 
Importantly, the performance advantage for \textsc{UCB-VI} comes at a significant computational cost. As shown in Table~\ref{tab:prophet_runtime_2}, both Algorithm~\ref{alg:sl} and Algorithm~\ref{alg:unifexp} are markedly more efficient. As $k$ increases, their running times  stay almost the same while the running time of \textsc{UCB-VI} scales significantly. 
This is because, in each episode, UCB-VI  computes the optimal value function of a MDP, which involves significant per-episode computation. In contrast, our algorithm only needs to record the aggregate reward observed in each episode, making it  more computationally efficient.

\bibliographystyle{plainnat}
\bibliography{references}

\end{document}